%% file: RobGen.tex
\documentclass[sn-mathphys-num]{sn-jnl}

\usepackage{graphicx}%
\usepackage{multirow}%
\usepackage{amsmath,amssymb,amsfonts}%
\usepackage{amsthm}%
\usepackage{mathrsfs}%
\usepackage[title]{appendix}%
\usepackage{xcolor}%
\usepackage{textcomp}%
\usepackage{manyfoot}%
\usepackage{booktabs}\let\cline\cmidrule%
\usepackage{algorithm}%
\usepackage{algorithmicx}%
\usepackage{algpseudocode}%
\usepackage{listings}%
\usepackage[utf8]{inputenc} 
\usepackage[T1]{fontenc}    
\usepackage{hyperref}       
\usepackage{url}            
\usepackage{amsfonts}       
\usepackage{nicefrac}       
\usepackage{microtype}      

\usepackage{caption}
\usepackage{subcaption}
\usepackage{mathtools}

\input{math_commands.tex}

\theoremstyle{thmstyleone}%
\newtheorem{theorem}{Theorem}
\theoremstyle{thmstyletwo}%
\newtheorem{remark}{Remark}%

\theoremstyle{thmstylethree}%
\newtheorem{definition}{Definition}%

\theoremstyle{plain}
\newtheorem{lemma}{Lemma}[section]

\newtheorem{assumption}{Assumption}[section]

\begin{document}

\title[Gentle Local Robustness implies Generalization]{Gentle Local Robustness implies Generalization}


\author*[1]{\fnm{Khoat} \sur{Than}}\email{khoattq@soict.hust.edu.vn}

\author[2]{\fnm{Dat} \sur{Phan}}\email{phandat12082002@gmail.com}

\author[1,3]{\fnm{Giang} \sur{Vu}}\email{lgv001@ucsd.edu}

\affil*[1]{\orgname{Hanoi University of Science and Technology}, 
\orgaddress{\state{Hanoi}, \country{Vietnam}}}

\affil[2]{\orgdiv{VinBigdata Institute}, \orgname{Vingroup}, \orgaddress{\state{Hanoi}, \country{Vietnam}}}

\affil[3]{\orgname{University of California}, \orgaddress{\city{San Diego}, \state{CA}, \country{USA}}}


\abstract{Robustness and generalization ability of machine learning models are of utmost importance in various application domains. There is a wide interest in efficient ways to analyze those properties. One important direction is to analyze connection between those two properties. Prior theories suggest that a robust learning algorithm can produce trained models with a high generalization ability. However, we show in this work that the existing error bounds are \textit{vacuous} for the Bayes optimal classifier which is the best among all measurable classifiers for a classification problem with overlapping classes. Those bounds cannot converge to the true error of this ideal classifier. This is undesirable, surprizing, and never known before.  We then present a class of novel bounds, which are \textit{model-dependent} and \textit{provably tighter} than the existing robustness-based ones. Unlike prior ones, our bounds are guaranteed to converge to the true error of the best classifier, as the number of samples increases. We further provide an extensive experiment and find that two of our bounds are often non-vacuous for a large class of deep neural networks, pretrained from ImageNet.}

\keywords{Model robustness, generalization ability, Error bound}



\maketitle

\section{Introduction}
Robust  learning algorithms \cite{xu2012robustnessGeneralize} can produce robust models which can resist small changes of data samples. Such an ability is crucial for modern applications, since non-robust models may face adversarial attacks \cite{madry2018towards,xu2020adversarial,zhou2022adversarial}.  A robust model not only can deal well with attacks but also can generalize well on unseen data.

In this work, we focus on analyzing the connection between robustness of a model and its generalization ability.  \citet{xu2012robustnessGeneralize} provided one of the very first theories to show that the models returned from a robust algorithm can generalize well on unseen data. Their robustness theory basically assumes that the learning algorithm must ensure a small deviation of the losses in areas around the training examples. This assumption is often known as \textit{algorithmic robustness}. This theory has been used in various contexts \cite{sokolic2017generalization,sokolic2017robustDNN,qi2013robustSVM,bellet2015robustnessMetricL,liu2017spectralClustering,li2021orthogonalDNN,shi2014sparse} where specific forms of robustness level ($\epsilon$) are provided. Recently, \citet{kawaguchi2022robustness} made a significant improvement in the uncertainty part which can bring  algorithmic robustness closer to practice. 

A major limitation of those algorithmic robustness-based theories is \textit{vacuousness}. For example,  for 0-1 loss,  an incorrect prediction of a  classifier can produce $\epsilon =1$ which equals robustness of the worst model. In practice, some incorrect predictions sometimes appear and may not be avoided, even for excellent models.  Theoretically, we show in subsection~\ref{sec-Vacuousness} that those theories are vacuous even for the Bayes classifier which is the best among all measurable classifiers for a classification problem with overlapping classes. This is undesirable.

Despite being really useful for evaluation and comparison between learning algorithms, those robustness-based bounds pose various difficulties to evaluate a specific model or compare two models. This fact limits the use of these error bounds in model selection. Furthermore, it is nontrivial \cite{xu2012robustnessGeneralize} to use those bounds to \textit{compare two models returned from different learning algorithms}, especially for stochastic algorithms  that are prevalent nowadays. These difficulties call for a novel model-dependent bound, which depends on a trained model only.

Our contributions in this work are as follows:
\begin{itemize}
\item We first point out the \textit{vacuousness} of the existing robustness-based bounds for the error of the best model among all measurable classifiers for a classification problem with overlapping classes. \textit{Those bounds cannot converge to the true error of the best model even for arbitrarily large number of training samples}. This is problematic and hence the use of those bounds to explain  generalization ability of an imperfect model is not well theoretically-justified.

\item We next present a novel class of error bounds, which are \textit{model-dependent}, by making a fine-grained analysis about local behaviors of a model at different small areas in the data space. \textit{Our bounds require no assumption on the model or learning algorithm}, but are  \textit{provably tighter} than previous  robustness-based  ones. 

\item For the best classifier, we show that our bounds converge to its true error as the number of samples increases. This suggests that our bounds resolve the major limitations of prior bounds and provide a significant step for the robustness approach.

\item We empirically compare those bounds on some real-life datasets and modern neural networks, and found that our bounds can reflect performance of a model better than the baselines. Furthermore, two of our bounds are often nonvacuous.
\end{itemize}

\textit{Roadmap:} The next section reviews the background about robustness-based bounds, discusses some of their issues, and closely related work. Section~\ref{sec-Local-behaviors-gen} presents our novel bounds and provides some theoretical comparisons. Section~\ref{sec-evaluation} presents our empirical evaluation, and Section~\ref{sec-Conclusion} concludes the paper. Details about experimental settings, proofs, and more experimental results appear in appendices.

\section{Backgrounds and related work} \label{sec-Backgrounds-related-work}

\textit{Notations:} A bold character (e.g., $\vz$) often denotes a vector, while a bold big symbol (e.g., $\mS$) often denotes a  set. Denote $\| \cdot \|$ as the $\ell_2$-norm. $| \mS |$ denotes the size/cardinality of a set $\mS$, and $[K]$ denotes the set $\{1, ..., K\}$ for a given integer $K \ge 1$. 

Consider a \textit{learning problem} specified by a model (or hypothesis) class $\gH$, an instance set $\gZ$, and a loss function $\ell: \gH \times \gZ  \rightarrow \mathbb{R}$. 
Given a distribution $P$ defined on $\gZ$, the quality of a model $\vh \in \gH$ is measured by its \textit{expected loss} $F(P, \vh) = \E_{\vz \sim P}[\ell(\vh,\vz)]$. In practice, we can collect a training set $\mS = \{\vz_1, ..., \vz_n\} \subseteq \gZ$ of size $n$ and often work with the \emph{empirical loss}  $F(\mS, \vh) =  \frac{1}{|\mS|} \sum_{\vz \in \mS} \ell(\vh,\vz)$. Quantity $F(P, \vh)$ tells the generalization ability of model $\vh$. A \textit{learning algorithm} $\gA$ will pick an $\gA_S \in \gH$ based on a given training set $\mS$.


Let $\Gamma(\gZ) := \bigcup_{i=1}^{K } \gZ_i$ be a partition of $\gZ$ into $K $ disjoint nonempty subsets. Denote $\mS_i = \mS \cap \gZ_i $, and $n_i = | \mS_i |$ as the number of samples  falling into $\gZ_i$, meaning that $n = \sum_{j=1}^K n_j$. Denote $\mT_S = \{ i \in [K ] : \mS \cap \gZ_i \ne \emptyset \}$.  Also denote $a_i(\vh) = \E_{\vz}[\ell(\vh,\vz) | \vz \in \gZ_i]$ for $i \in [K]$.

\subsection{ Robustness-based bounds}

When studying generalization ability of a  model $\vh$, it is natural to  consider  the expected loss $ F(P, \vh)$. However, an accurate estimation of  $ F(P, \vh)$ is highly challenging, especially for complex models. One well-known way is to study the training algorithm that produces $\vh$. 

Denote $\gA_S$ as the model (or hypothesis) which is learned by an algorithm $\gA$ from a training set $\mS$ with $n$ samples. \citet{xu2012robustnessGeneralize} defined the following.

\begin{definition}\label{def-algorithmic-robustness}
A learning algorithm $\gA$ is $(K, \epsilon)$-robust, for $ K \in \sN$ and $\epsilon: \gZ^n \rightarrow \R$, if $\gZ$ can be partitioned into $K$ disjoint sets, denoted by $\{\gZ_k\}^K_{k=1}$, such that the following holds for all $\mS \in \gZ^n: \forall \vs \in \mS, \forall \vz \in \gZ$, if $\vs, \vz \in \gZ_k$ for some index $k$, then $| \ell(\gA_S, \vs) - \ell(\gA_S, \vz) | \le \epsilon(\mS)$.
\end{definition}

Basically, algorithm $\gA$ is robust if every model learned by $\gA$ is robust on areas around the given training samples, according to a loss function $\ell$. This suggests that the trained model can generalize well on areas around training samples.  In order to formalize connection between  robustness of a learning algorithm and generalization of a trained model, we need the following assumption.

\begin{assumption}[Algorithmic robustness]\label{assumption-Alg-robust}
The learning algorithm $\gA$ is $(K, \eps)$-robust.
\end{assumption}

\citet{xu2012robustnessGeneralize} provided the following  bound about the expected loss of a  model learned by a robust algorithm.

\begin{theorem}[\cite{xu2012robustnessGeneralize}] \label{thm-Alg-robust-gen-Xu12}
Given Assumption~\ref{assumption-Alg-robust}, consider $\vh$ learned by algorithm $\gA$ from a dataset $\mS$ which consists of $n$ i.i.d. samples from distribution $P$, and a bounded loss $\ell$.  For any  $\delta >0$, denote $C_{\gH} = \sup_{\vf \in \gH, \vz \in \gZ} \ell(\vf,\vz)$ and $g_1(K, \mS, \delta) = C_{\gH} \sqrt{\frac{2 K \ln 2 - 2 \ln(\delta)}{n}}$. With probability at least $1-\delta$: 
\begin{equation}
\label{eq-thm-Alg-robust-gen-Xu12}
 F(P, \vh) \le g_1(K, \mS, \delta) + F(\mS, \vh)  + { \epsilon(\mS) }  
 \end{equation}
\end{theorem}

It is easy to see that when both $ \epsilon(\mS)$ and empirical loss $F(\mS, \vh)$ are small, the expected loss $F(P, \vh)$ is also small, implying that model $\vh$  generalizes well on unseen data. This suggests that a robust learning algorithm may return models with high generalization ability. The reverse however is not true. A model with good generalization ability may not come from a robust learning algorithm. 

By analyzing concentration of a multinomial random variable, \citet{kawaguchi2022robustness} can replace $g_1$ in Theorem~\ref{thm-Alg-robust-gen-Xu12} with a significantly smaller quantity $g_2$.

\begin{theorem}[\cite{kawaguchi2022robustness}] \label{thm-Alg-robust-gen-Kawa22}
Given the assumption and notations as in Theorem \ref{thm-Alg-robust-gen-Xu12}.  For any  $\delta >0$, denote $g_2(K, \mS, \delta) = C (\sqrt{2} + 1) \sqrt{\frac{| \mT_S| \ln(2 K /\delta)}{n}} + \frac{ 2 C | \mT_S| \ln(2 K /\delta)}{n}$, where $C = \sup_{\vz \in \gZ} \ell(\vh,\vz)$. The following holds  with probability at least $1-\delta$: 
\begin{equation}
\label{eq-thm-Alg-robust-gen-Kawa22}
  F(P, \vh) \le g_2(K, \mS, \delta) + F(\mS, \vh) + { \epsilon(\mS) } 
 \end{equation}
\end{theorem}

Compared with $g_1$, the new uncertainty term $g_2$ can be significantly smaller, since it does not depend on the whole model family and logarithmically depends on $K$. This is an exponential improvement, and can help  bound (\ref{eq-thm-Alg-robust-gen-Kawa22}) to be more practical than bound (\ref{eq-thm-Alg-robust-gen-Xu12}). \citet{kawaguchi2022robustness} further showed that $g_2$ can be improved by $g_3$, where $a_o = \max_{j \notin \mT_S} a_j(\vh)$ and 
\begin{equation}
\label{eq-Alg-robust-gen-Kawa22-tight-uncertainty}
\small{
g_3(K, \mS, \delta) = \frac{\sqrt{\ln(2K/\delta)}}{n} \sum\limits_{i \in \mT_S} \sqrt{n_i} \left(a_o + \sqrt{2} a_i(\vh)\right)   + \frac{2\ln(2K/\delta)}{n} (a_o | \mT_S| +  \sum\limits_{i \in \mT_S} a_i(\vh) ) }
\end{equation}

While \citet{kawaguchi2022robustness} made a significant progress for the connection between  algorithmic robustness and generalization by improving the uncertainty part ($g_1$), the role of robustness level ($\epsilon$) is kept unchanged. There remains a serious issue of those bounds, as discussed below.

\subsection{The vacuousness issue and its main origins} \label{sec-Vacuousness}

Consider a model $\vh$ returned by a robust algorithm $\gA$. Definition~\ref{def-algorithmic-robustness} implies that $\eps(\mS) \ge \sup_{i \in \mT_S} \eps_i(\vh) $, where $\eps_i(\vh) = \sup_{\vz' \in \mS_i, \vz \in \gZ_i}  | \ell(\vh, \vz') -\ell(\vh, \vz) |$. This fact can make the bounds (\ref{eq-thm-Alg-robust-gen-Xu12}) and (\ref{eq-thm-Alg-robust-gen-Kawa22}) vacuous even for extremely good models. For example, for 0-1 loss $\ell$ and a binary classifier,  an incorrect prediction for one example can lead to $\epsilon(\mS) =1$ which equals robustness of the worst model. In practice, it is common and acceptable to have some incorrect predictions from good models. 

To see the seriousness of this limitation, consider a classification problem with overlapping classes and the Bayes optimal classifier which is the best among all measurable classifiers. Note that the Bayes classifier is ideal for this problem, and no classifier found in practice can be better. To formally define overlapping and Bayes classifier, we first define the   $g$-margin of an instance $\vs =(\boldsymbol{x},y)$ according to a classifier $g$ to be 
$\gamma(\vs, g) = \sup \{\nu: \|\boldsymbol{x} - \boldsymbol{x}'\| \le \nu \Rightarrow g(\boldsymbol{x}') = y, \forall \boldsymbol{x}'\}$. This definition of instance margin comes from \cite{sokolic2017robustDNN}. Let $\mathcal{H}_B$ be the set of all measurable classifiers defined on $\gZ$. We can define the classifier-agnostic \textit{margin} of each instance  and the overlapping area  as 
\begin{align}
\gamma_{\text{in}}(\vs,  \mathcal{H}_B) &= \sup \{\gamma(\vs, g) : g \in \mathcal{H}_B\}  & \text{(Instance margin)}&\\
\gO &= \{\vs \in \gZ :  \gamma_{\text{in}}(\vs,  \mathcal{H}_B) = 0\} & \text{(Overlapping area)}&
\end{align}

\begin{theorem}[Bayes optimal classifier] \label{thm-Bayes-classifier}
Consider a classification problem with distribution $P$ supported in a continuous set $\gZ$, an overlapping area $\gO \subseteq \gZ$,   the 0-1 loss function $\ell$, and any partition $\Gamma_o = \gO_1 \cup \gO_2 \cup \cdots \cup \gO_N$ of $\gO$ into finite number of subsets.  Let $\vh^* = \arg\min_{\vh \in  \gH_B} F(P, \vh)$ be  the Bayes optimal classifier, and $\eps_o(\vh, \gV) = \sup_{\vz', \vz \in \gV}  | \ell(\vh, \vz') -\ell(\vh, \vz) |$. If $\gO$ has non-zero measure, i.e., $P(\gO) > 0$, then $ F(P, \vh^*) \le P(\gO)$ and $ \eps_o(\vh^*, \gO) = 1$ and $ \sup_{k \in [N]} \eps_o(\vh^*, \gO_k) = 1$.
\end{theorem}

Various implications can be derived from this theorem whose  proof appears in Appendix~\ref{app-Proof-vacuousness}. Firstly, the definition of robustness level $\epsilon$ in Definition~\ref{def-algorithmic-robustness} is the main cause for vacuousness. Indeed, for a classification problem with overlapping classes and its Bayes optimal classifier $\vh^*$, Theorem~\ref{thm-Bayes-classifier} shows  the robustness level $\epsilon=1$ which can be very far from the true loss of $\vh^*$. When the overlapping area is sufficiently small, meaning $ F(P, \vh^*) \approx 0$, robustness level $\epsilon$ makes the bounds (\ref{eq-thm-Alg-robust-gen-Xu12}) and (\ref{eq-thm-Alg-robust-gen-Kawa22}) vacuous. Secondly, Theorem~\ref{thm-Bayes-classifier} further suggests that vacuousness still happens for the best partition $\Gamma^*$. This means there is no hope to avoid vacuousness by optimizing the bounds (\ref{eq-thm-Alg-robust-gen-Xu12},\ref{eq-thm-Alg-robust-gen-Kawa22}) according to $\Gamma$. Thirdly, vacuousness appears not only in $\vh^*$ but all other classifiers for this classification problem.

\textbf{The main origins of vacuousness:}
As discussed before, the definition of robustness level $\epsilon$ in Definition~\ref{def-algorithmic-robustness} is the main cause for vacuousness in prior robustness-based bounds. More specifically, Definition~\ref{def-algorithmic-robustness} implicitly requires two specific operations:
\begin{itemize}
    \item \textit{Supremum:} Note that  $\eps(\mS) \ge \sup\limits_{i \in \mT_S} \eps_i(\vh) $, where $\eps_i(\vh) = \sup\limits_{\vz' \in \mS_i, \vz \in \gZ_i}  | \ell(\vh, \vz') -\ell(\vh, \vz) |$. This means in order to compute the robustness level $\eps(\mS)$, we must take supremum operation from local robustness levels ($\eps_i$) in all local regions that contain some examples in $\mS$. Therefore $\eps(\mS)$ can be considered as measuring the \textbf{Global robustness} of a model returned by $\gA$. Such an operation will lead to vacuousness when there exists a vacuous event in a local region, e.g., $\eps_i(\vh) =1$ for 0-1 loss. Note that it is common in practice that a trained model may have some wrong predictions. In those cases, the supremum operation will lead to vacuousness in prior bounds.
    \item \textit{Stochasticity inclusion:} When algorithm $\gA$ is stochastic (which is prevalent in practice, e.g., SGD), different runs may return different trained models even for the same training set $\mS$ and parameter setting. As a result, we need to take stochasticity of algorithm $\gA$ into computation of $\eps(\mS)$. So in fact 
    \[\eps(\mS) \ge \sup\limits_{\eta, i} \eps_i(\gA_S(\eta))\] 
    where $\gA_S(\eta)$ denotes the model returned by $\gA$ given a dataset $\mS$, $\eta$ denotes the source that causes stochasticity for algorithm $\gA$. This fact suggests that stochastic source $\eta$ plays a crucial role in the definition of $\eps(\mS)$. Hence, an extensive investigation about $\eps(\mS)$ requires us to take all stochastic sources into consideration, which is intractable in practice. More importantly, some runs of algorithm $\gA$ can produce an imperfect model, which can make some wrong predictions. This suggests that $\eps(\mS)$ easily is vacuous in practice.
\end{itemize}

\subsection{Related work}

Although widely being used in many contexts, the theoretical advancement for robustness-based bounds is quite slow. \citet{kawaguchi2022robustness} made a significant progress to improve the uncertainty term in (\ref{eq-thm-Alg-robust-gen-Xu12}). However, to the best of our knowledge, no prior study significantly improves the robustness term nor removes the assumption on the learning algorithm.  The vacuousness issue of prior bounds is problematic, and hence those bounds cannot be used to explain the success of the models in practice. Our work tackles those limitations to derive novel bounds that are practical. 

A closely related work \cite{hou2023instanceGen} uses optimal transport to provide a model-dependent bound. Ignoring some constants, \citet{hou2023instanceGen} showed that  $ F(P, \vh) \le  F(\mS, \vh) + \gamma L_{\ell} \sum_{i \in \mT_D} \frac{n_i}{n} \max\{1, L_{\vh,i} \}  + \gamma L_{\ell} \max \{1, L_{\vh}\}  \sqrt{(\log 4 - \log \delta)/n}  +  \sqrt{K /n} $, provided that $\ell(\cdot, \vz)$ is $L_{\ell}$-Lipschitz continuous and $\vh$ is $L_{\vh}$-Lipschitz continuous w.r.t  its input, where  $L_{\vh,i}$ is the local Lipschitz constant of $\vh$ at area $\gZ_i$ and $\gamma$ is the maximal diameter of the $\gZ_i$'s. Note that the term $\sqrt{K/n}$ causes their bound to surfer from the curse of dimensionality, since $K = \gamma^{-O(v)}$ in the worst case where the input space has $v$ dimensions. Hence their bound is significantly inferior to  ours. We further point out in Subsection~\ref{subsec-Examples} that their bound is inferior to ours due to the global Lipschitz constant.

To analyze  generalization ability, various approaches have been studied, including  Radermacher complexity \cite{bartlett2002RC,golowich2020RC}, algorithmic stability \cite{shalev2010StabilityLearnability,feldman2019stability}, algorithmic robustness \cite{xu2012robustnessGeneralize,kawaguchi2022robustness}, PAC-Bayes \cite{mcallester1999PACBayes,haddouche2023pac,biggs2023tighterPAC,zhou2019CompressionBound,arora2018strongerBounds}, local Lipschitzness \cite{hou2023instanceGen}. 
Some studies \cite{bartlett2017SpectralMarginDNN,golowich2020RC,wei2019dataRC} use Rademacher complexity to provide data- and model-dependent bounds as ours. Their focus is on analyzing generalization ability of deep neural networks (DNNs). One remaining issue is that their bounds depend on the norm of weight matrices in a DNN which is often huge for practical DNNs \cite{arora2018strongerBounds,wei2019dataRC}. In contrast, our bounds use local information of a model and hence can provide tighter estimates for its expected loss. 

PAC-Bayes bounds \cite{mcallester1999PACBayes,haddouche2023pac,biggs2023tighterPAC} recently has received great attention, and provide  non-vacuous bounds \cite{zhou2019CompressionBound,mustafa2024non} for some DNNs. Those bounds  often estimate $\E_{\mathring{\vh}}  [F(P, \mathring{\vh})]$ which is the expectation over the distribution of $\mathring{\vh}$. It means that those bounds are for a \textit{stochastic model} $\mathring{\vh}$. Hence they provide limited understanding for a specific deterministic model $\vh$. \citet{neyshabur2018SpectralMarginDNN} provided an attempt to derandomization for PAC-Bayes but resulted in vacuous bounds for modern neural networks \cite{arora2018strongerBounds}. 
On the other hand,  stability-based bounds \cite{shalev2010StabilityLearnability,feldman2019stability} connect the stability of a learning algorithm with generalization ability. Despite having some interesting properties, stability-based bounds are inferior to the  robustness-based bound in \cite{kawaguchi2022robustness} in some situations.

\section{Local behaviors and generalization} \label{sec-Local-behaviors-gen}

In this section, we develop a class of novel bounds that connect local behaviors with generalization ability of a specific model. Our bounds do not require the strict assumptions as prior bounds, are model-specific and data-dependent. We show that our bounds are provably tighter than the prior ones.

\subsection{Upper bounds}

As discussed in the previous section, algorithmic robustness-based bounds use $\eps(\mS)$ to globally quantify the robustness of a model over the whole data space. This quantity summarizes the local robustness of a model  at different small regions in a simple way, by using a supremum operation. Therefore, using this quantity will often make the bound vacuous, as pointed out before. 

We overcome this limitation of prior studies by incorporating the local robustness of a model  at different small regions into generalization bound. To this end, we make a finer-grained analysis than \cite{xu2012robustnessGeneralize}. The following theorems summarize the results, whose proofs appear in Appendix~\ref{app-Proofs-for-main-results}.

\begin{theorem}[Local Robustness] \label{thm-Local-Robustness-generalization}
Consider a model $\vh$ learned from a dataset $\mS$ with $n$ i.i.d. samples from distribution $P$, and a bounded loss $\ell$.  For each $i \in [K]$, let $\epsilon_i(\vh) = \sup\limits_{\vs \in \mS_i,  \vz \in \gZ_i} | \ell(\vh, \vs) - \ell(\vh,\vz) |$.  For any  $\delta >0$, denoting $g_2(K, \mS, \delta)$ as in Theorem~\ref{thm-Alg-robust-gen-Kawa22},  with probability at least $1-\delta$: 
 \begin{equation}
 \label{eq-thm-Local-Robustness-generalization}
   F(P, \vh) \le g_2(K, \mS, \delta) + F(\mS, \vh) +   {  \sum_{i \in \mT_S}  \frac{ n_i}{n}  \epsilon_i(\vh) }  
\end{equation}
\end{theorem}

This theorem shows that the expected loss of a model can be bounded by using $\epsilon_i(\vh)$ which describes local robustness of $\vh$ at different regions. It suggests that a model can generalize well  when it is ``locally robust'' at  different small regions. A model can have a small expected loss over the whose sample space if it is locally robust and has a small training loss $F(\mS, \vh)$.

There are three main differences between our bound in Theorem~\ref{thm-Local-Robustness-generalization} and the bounds in Theorems~\ref{thm-Alg-robust-gen-Xu12} and \ref{thm-Alg-robust-gen-Kawa22}. Firstly, our bound (\ref{eq-thm-Local-Robustness-generalization})  \textit{does not require the strict assumption of algorithmic robustness}. This is a significant advantage. Secondly, our bound (\ref{eq-thm-Local-Robustness-generalization}) is \textit{model-specific and data-dependent}, since it depends on a specific model $\vh$ and training sample $\mS$ only.  This is a big advantage over prior bounds, and enables us to do model selection or compare different trained models. This advantage is really beneficial in practice. Thirdly, the global robustness level $\eps(\mS)$ in the bound (\ref{eq-thm-Alg-robust-gen-Kawa22}) is replaced with a finer quantity ${  \sum_{i  \in \mT_S}  \frac{ n_i}{n}  \epsilon_i(\vh) }$, removing the serious issue of ``stochasticity inclusion'' in prior bounds. This is a big advantage and helps robustness-based bounds less vacuous and closer to practice. 

Next, we present another bound which considers the average-case robustness.

\begin{theorem} \label{thm-Local-Sensitivity-generalization}
Given notations in Theorem~\ref{thm-Local-Robustness-generalization}, denoting $\bar{\epsilon}_i(\vh) = \frac{1}{n_i} \sum_{\vs \in \mS_i} \E_{\vz \in \gZ_i} | \ell(\vh,\vz) -  \ell(\vh,\vs) | $ for each index $i \in \mT_S$,  with probability at least $1-\delta$: 
 \begin{equation}
 \label{eq-thm-Local-Sensitivity-generalization}
   F(P, \vh) \le  g_2(K, \mS, \delta) + F(\mS, \vh) +   {  \sum_{i \in \mT_S}  \frac{n_i}{n} \bar{\epsilon}_i(\vh) }  
\end{equation}
\end{theorem}

The use of supremum (or max) operation to define robustness in prior bounds and in (\ref{eq-thm-Local-Robustness-generalization}) suggests that we are considering the worst-case robustness (or sensitivity) of the loss at every local  region of the data space. Such a consideration is really strict, and easily lead to vacuous bounds. This is evidenced in our experiments for a large class of neural networks, presented in Section~\ref{sec-evaluation}. To avoid such situations, Theorem~\ref{thm-Local-Sensitivity-generalization} provides a finer-grained analysis about the loss. It says that a model $\vh$ can generalize well if its average robustness (or sensitivity, measured by $\bar{\epsilon}_i$) is small at at every local  region of the data space.

This bound can be meaningful even for the cases that few local robustness levels ($\bar{\epsilon}_i$) are large at some small input areas. However, when model $\vh$ is non-robust at most of the local areas, the sum $\sum_{i \in \mT_S}  \frac{n_i}{n} \bar{\epsilon}_i(\vh)$ can be large and hence our bounds can be vacuous. The same behaviors also appear in the bounds in Theorems \ref{thm-Local-Robustness-generalization} and \ref{thm-Local-average-generalization}. Of course, those cases happen for very bad models.

\textit{Trade-off:} It is worth observing that there is a trade-off between the empirical loss $F(\mS, \vh)$ and the robustness term $\sum_{i \in \mT_S}  \frac{n_i}{n} \epsilon_i(\vh)$ (and $\sum_{i \in \mT_S}  \frac{n_i}{n} \bar{\epsilon}_i(\vh)$). A very robust  model can make  $\sum_{i \in \mT_S}  \frac{n_i}{n} {\epsilon}_i(\vh)$ small, but may not be flexible enough to fit the training set $\mS$. It suggests that a too robust model can have a large training loss. On the other hand, a small empirical loss $F(\mS, \vh)$ often requires $\vh$ to have a high capacity. Such a model may be non-robust at some local areas, meaning that some $\epsilon_i$ can be large. An evidence can be seen from 4th column of Table~\ref{tab:imagenet-uncertainty-comparison}, where some $\epsilon_i$'s are large even for good models with high accuracy.

The final bound incorporates the averages of the loss at local regions.

\begin{theorem} \label{thm-Local-average-generalization}
Given notations in Theorem~\ref{thm-Local-Robustness-generalization},  with probability at least $1-\delta$: 
 \begin{equation}
\label{eq-thm-Local-average-generalization}
   F(P, \vh) \le  g_2(K, \mS, \delta) + {  \sum_{i \in \mT_S}  \frac{n_i}{n} a_i(\vh) }  
\end{equation}
\end{theorem}

This result tells that the expected loss of a model can be bounded by a convex combination of some expected losses over some small regions. This is intuitive. An interesting point is that this bound does not require access to the empirical loss, which may be beneficial in some contexts, where access to the training dataset is impossible.\footnote{One example of such situations is the evaluation of a publicly  pretrained model which was trained from a private huge dataset. Pretrained models, such as large language models, are prevalent nowadays and freely available to be used in many different tasks. The need of evaluation of those models in different contexts is of significant interest.}

\subsubsection{Tightness}

We have already presented some novel bounds that show the significant role of local behaviors  at different small regions to the generalization ability of a model. Theorem~\ref{thm-Local-Robustness-generalization} shows that the expected error of a model can be estimated by a convex combination of (worst-case) robustness levels at different small regions, while Theorem~\ref{thm-Local-Sensitivity-generalization} uses the average-case robustness (sensitivity). Theorem~\ref{thm-Local-average-generalization} replaces robustness level by the averages of the loss at small regions around the training examples.  The next lemma points out the tightness of our bounds, whose proof appears in Appendix~\ref{app-proof-bound-comparison}.

\begin{lemma}\label{lem-bound-compare}
With notations in Theorems \ref{thm-Alg-robust-gen-Kawa22},  \ref{thm-Local-Robustness-generalization} and  \ref{thm-Local-Sensitivity-generalization}: 
 \begin{align}
 \nonumber
  \sum_{i \in \mT_S}  \frac{n_i}{n} a_i(\vh)  \le F(\mS, \vh) +   {  \sum_{i \in \mT_S}  \frac{n_i}{n} \bar{\epsilon}_i(\vh) }  
   \le  F(\mS, \vh) + \sum_{i \in \mT_S}  \frac{ n_i}{n} \epsilon_i(\vh) \le  F(\mS, \vh)  + \eps(\mS) 
\end{align}
\end{lemma}

This lemma suggests that our new bounds are tighter than prior one in (\ref{eq-thm-Alg-robust-gen-Kawa22}). It is easy to observe that $ a_i(\vh) <  F(\mS_i, \vh) +  \bar{\epsilon}_i(\vh)    <  F(\mS_i, \vh) +  \epsilon_i(\vh)$ when the loss of $\vh$ is not constant in area $\gZ_i$. This is practical and suggests that our bounds (\ref{eq-thm-Local-Sensitivity-generalization},\ref{eq-thm-Local-average-generalization}) are often strictly tighter than prior bound (\ref{eq-thm-Alg-robust-gen-Kawa22}). Finally, it is worth noticing that  the uncertainty term $g_2$ can be replaced by $g_3$  in (\ref{eq-Alg-robust-gen-Kawa22-tight-uncertainty}) to make our bounds tighter.

\subsubsection{Non-vacuousness for the Bayes optimal classifier} 

Return to the  problem and Bayes optimal classifier in Theorem~\ref{thm-Bayes-classifier} with partition $\Gamma = \gZ_1 \cup \gZ_2$ where $\gZ_1 = \gO$ and $\gZ_2 = \gZ \setminus \gO$. Prior bounds  have $\epsilon(\mS) = \epsilon_o(\vh^*, \gZ_1)=1$, whenever $\mS \cap \gO \ne \emptyset$, which is vacuous. However,  Theorem~\ref{thm-Local-Robustness-generalization} replaces $\epsilon$ by 
\begin{equation}
\epsilon_{local} = \frac{ n_o}{n}  \epsilon_o(\vh^*, \gZ_1) + \frac{ n- n_o}{n}  \epsilon_o(\vh^*, \gZ_2) = \frac{ n_o}{n}  \epsilon_o(\vh^*, \gZ_1) \le \frac{ n_o}{n}    
\end{equation}
where $n_o$ is the number of samples of $\mS$ occurring in area $\gZ_1$. We have the following observation, whose proof appears in Appendix~\ref{app-Proof-vacuousness}.

\begin{lemma}\label{lem-Bayes-classifier-local}
Consider the classification problem in Theorem~\ref{thm-Bayes-classifier}. Let $\mS$ contain $n$ i.i.d. samples from distribution $P$. For any $\delta \ge 2e^{-n/4}$,  $\Pr\left(\frac{n_o}{n} \le P(\gO) + \sqrt{\frac{\ln(2/\delta)}{n}} \right) \ge 1- \delta$.
\end{lemma}

This simple lemma shows that with a high probabilty, $\epsilon_{local} \le P(\gO) + \sqrt{\frac{\ln(2/\delta)}{n}}$, which will goes to the measure $P(\gO)$ of the overlapping area, as $n$ goes to infinity. When such an area $\gO$ is small, $P(\gO)$ can be small, and hence our bounds in (\ref{eq-thm-Local-Robustness-generalization}, \ref{eq-thm-Local-Sensitivity-generalization}, \ref{eq-thm-Local-average-generalization}) can be meaningful even for the cases that prior bounds are vacuous. 

It is worth mentioning that our bounds can lead to a tight error bound for the Bayes optimal classifier. Indeed, observe that $\Pr(\vh^*(\vx) \ne y) = F(P, \vh^*)$, for 0-1 loss, where each input $\vx$ has its true label $y$. Furthermore, for the partition $\Gamma$, it is easy to see that $F(\mS, \vh^*) \le {n_o}/{n}, |\mT_S| \le 2, a_2(\vh^*) =0, a_1(\vh^*) \le 1$. Therefore, $g_2(2, \mS, \delta) \le 3 \sqrt{\frac{2 \ln(4 /\delta)}{n}} + \frac{ 4 \ln(4/\delta)}{n}$. Theorem~\ref{thm-Local-average-generalization} suggests that
\begin{equation}
\Pr(\vh^*(\vx) \ne y) \le 3 \sqrt{\frac{2 \ln(4 /\delta)}{n}} + \frac{ 4 \ln(4 /\delta)}{n} + \frac{n_o}{n} a_1(\vh^*)
\end{equation}
This  is a  tight bound for the error of the Bayes optimal classifier. As $n \rightarrow \infty$, this upper bound will go to $P(\gO) a_1(\vh^*)$, which is the true error of $\vh^*$. Such a bound may be useful elsewhere.

\begin{remark}
This simple analysis provides a crucial implication. The existing bounds, which are based on algorithmic robustness or the global quantity $\epsilon(\mS)$, always produce $\epsilon(\mS)=1$ and hence cannot reflect well the true error of the Bayes classifier even for arbitrarily large $n$. This also happens for any imperfect classifier, and is problematic. Therefore, the use of those bounds to support or explain imperfect classifiers \cite{sokolic2017generalization,sokolic2017robustDNN,qi2013robustSVM,bellet2015robustnessMetricL,liu2017spectralClustering,li2021orthogonalDNN,shi2014sparse} is not well theoretically-justified. In contrast, our bounds are guaranteed to converge to the true error. This is truly beneficial.
\end{remark}

\subsection{Lower bounds}
We next consider lower bounds  for the expected loss of a model. Those bounds sometimes are of interest, but our results  before are upper bounds. \citet{xu2012robustnessGeneralize} already suggested that $F(P, \vh) \ge  F(\mS, \vh)  - \eps(\mS)  - g_1(K, \mS, \delta) $. It is easy to see that this lower bound is meaningless in the cases of classification problems with overlapping classes, where the training loss $F(\mS, \vh)$ can be small but  $\eps(\mS) $ is large. One example is the Bayes optimal classifier as discussed in Theorem~\ref{thm-Bayes-classifier}. Moreover, this lower bound is $O(-\sqrt{K})$ which easily is vacuous for large $K$. While the uncertainty term of the upper bound was exponentially improved by \cite{kawaguchi2022robustness}, there has been no such improvement for the lower bound.

For any model class $\gH$, we obtain the following lower bounds, whose proofs appear in Appendix~\ref{app-proof-lower-bound}. 

\begin{theorem} \label{thm-Local-Rob-generalization-lower-bound}
Consider a family $\gH$ and a dataset $\mS$ with $n$ i.i.d. samples from distribution $P$, and a bounded loss $\ell$. Denote $ \bar{a}_i = \E_{\vh \in \gH}[a_i(\vh)], \hat{a} = \max_{j \in [K]}  \bar{a}_j$, and $\beta= 2\sum_{j =1}^{K}  P(\gZ_j) \bar{a}^2_j$, for each $i \in [K]$. If $\beta>0$, then for all $\delta \ge \exp\left(-{n\beta}/{(2\hat{a}^2)}\right)$,   the following holds  with probability at least $1-\delta$: \\
$\E_{\vh \in \gH} [F(P, \vh)] \ge \left( \sqrt{\sum_{i \in \mT_S}  \frac{n_i}{n} \bar{a}_i + \frac{1}{n} \hat{a}\ln(1/\delta)} - \sqrt{\frac{1}{n} \hat{a}\ln(1/\delta)} \right)^2
$ 
\end{theorem}

This theorem provides a lower bound on the average loss of the whole family $\gH$. When family $\gH$ only contains the models returned from a learning algorithm $\gA$, this theorem in fact provides a lower bound on the average error ($ \E_{\gA_S} [F(P, \gA_S)] $) of the predictions by $\gA$. It is worth mentioning that $ \E_{\gA_S} [F(P, \gA_S)] $ is often the main focus in the existing algorithmic stability-based theories \cite{shalev2010StabilityLearnability,lei2020stability}. As a result,  Theorem~\ref{thm-Local-Rob-generalization-lower-bound} provides a lower bound on the error of stable learning algorithms. On the other hand, PAC-Bayes theories \cite{mcallester1999PACBayes,biggs2022nonvacuousBound} often provide upper bounds for $\E_{\vh} [F(P, \vh)]$. More importantly, our lower bound does not depend on $K$ and is nonvacuous. 

\begin{theorem} \label{thm-Local-Rob-generalization-lower-bound-best}
Consider the best model $\vh^*$ with $F(P, \vh^*) = \min_{\vh \in  \gH} F(P, \vh)$ and a dataset $\mS$ with $n$ i.i.d. samples from distribution $P$, and a bounded loss $\ell$. Denote $\hat{a} = \max_{j \in [K]}  a_j(\vh^*)$ and $\beta= 2\sum_{j =1}^{K}  P(\gZ_j) a_j(\vh^*)^2$. If $\beta>0$ then for any $\delta \ge \exp\left(-{n\beta}/{(2\hat{a}^2)}\right)$,   the following holds  with probability at least $1-\delta$:  \\ 
$   F(P, \vh^*) \ge \left( \sqrt{\sum_{i \in \mT_S}  \frac{n_i}{n} a_i(\vh^*) + \frac{1}{n} \hat{a}\ln(1/\delta) } - \sqrt{\frac{1}{n} \hat{a}\ln(1/\delta)} \right)^2
$ 
\end{theorem}

The assumption $\beta>0$ in this theorem means that our learning problem is hard, since the best member in family $\gH$ still has some errors. In this case, the lower bound for the loss of $\vh^*$ is also the lower bound for the whole family. When $\gH \equiv \gH_B$, $\vh^*$ is the Bayes optimal predictor and $\beta>0$ means our learning problem is not learnable \cite{shalev2010StabilityLearnability}. In this case, Theorem~\ref{thm-Local-Rob-generalization-lower-bound-best} provides a lower bound for the error of all measurable predictors, which can be useful in some contexts.  

\subsection{About computing our bounds} \label{subsec-computing}

Although having significant advantages and being more practical than prior ones, our bounds in Theorems \ref{thm-Local-Robustness-generalization}, \ref{thm-Local-Sensitivity-generalization}, and \ref{thm-Local-average-generalization} are intractable to compute exactly. The main reason comes from the unknown form of the data distribution $P$ and the infinity/uncountability of the data space. Such unknown facts pose challenges for exactly computing the robustness/sensitivity levels, e.g., $\epsilon_i, \bar{\epsilon}_i, a_i$ in our bounds. Therefore, we need to approximate those quantities using a dataset.

Among those, bound (\ref{eq-thm-Local-average-generalization}) seems to be cheapest to compute. It requires $O(n)$ evaluations of the loss, for a dataset $\mS$ with $n$ samples, and also $O(n)$ arithmetic operations to count all $n_i$'s. Both bounds (\ref{eq-thm-Local-Robustness-generalization}) and (\ref{eq-thm-Local-Sensitivity-generalization}) require another dataset $\mD$ to approximate $\epsilon_i$ and $\bar{\epsilon}_i$. Each $\epsilon_i$ (and also $\bar{\epsilon}_i$) requires $O(n_i + m_i)$ evaluations of the loss and $O(n_i m_i)$ arithmetic operations. So in the worst case, one may need $O(n m)$ arithmetic operations to approximate our bounds. This is expensive for the cases that the datasets are large.

\section{Empirical evaluation} \label{sec-evaluation}

In this section, we empirically evaluate our bounds and compare with the baselines on modern DNNs and real-life datasets. Two evaluations include (1) Pre-trained models on ImageNet and (2) Trained models on moderate datasets. Two types of learning problems are used: classification task, and dimensionality reduction with PCA. More evaluations on average-size datasets and some classical models appear in Appendix~\ref{app-Experimental-setup-results}.

\subsection{Evaluation for publicly pre-trained models on ImageNet} \label{sec-evaluation-Large-scale}

\textit{Setup:} 20 pytorch trained models\footnote{https://pytorch.org/vision/stable/models.html} are used in our experiment. They are variants from 4  modern NN architectures: ResNet, VGG, DenseNet, and Swin Transformer. Some models have more than \textbf{150 layers},  some have \textbf{143.7M parameters}. They were well pretrained from ImageNet  with 1,281,167 images. 50,000 images from the ImageNet validation set is further used to compute the bounds. 

To evaluate the bounds (\ref{eq-thm-Alg-robust-gen-Kawa22},\ref{eq-thm-Local-Robustness-generalization},\ref{eq-thm-Local-Sensitivity-generalization},\ref{eq-thm-Local-average-generalization}), following \cite{kawaguchi2022robustness}, we partition  the input space into 10,000 areas, by  choosing randomly 10,000 images from the validation set to be centroids.  Based on those centroids, we can use K-means to assign the training images into different areas. Note that one can optimize this step to get better bounds. We next  approximate quantities $\epsilon, \epsilon_i, \bar{\epsilon}_i, a_i$ for each area $\gZ_i$, and $g_3$ in (\ref{eq-Alg-robust-gen-Kawa22-tight-uncertainty}) by using the validation set. The 0-1 loss function is used in our evaluation. Therefore, any bound beyond 1 will be vacuous. The result  for each model is  averaged from 5 random runs.

\textit{Result:} Table \ref{tab:imagenet-bounds-comparison} summarizes the results. We observe that prior bound (\ref{eq-thm-Alg-robust-gen-Kawa22}) is vacuous for every case, which is not surprised. Our bound (\ref{eq-thm-Local-Robustness-generalization}) behaves very similarly with prior bound. One reason may be that the partition $\Gamma$ with 10,000 areas seems too coarse in this case, and that the max operation is not good. On the other hand, our bound (\ref{eq-thm-Local-Sensitivity-generalization}) uses the mean operation which produces significantly better results. Both bounds (\ref{eq-thm-Local-Sensitivity-generalization},\ref{eq-thm-Local-average-generalization}) are non-vacuous in all cases. Furthermore, one  can easily observe the strong correlation between those bounds with the test accuracy. Our bounds (\ref{eq-thm-Local-Sensitivity-generalization},\ref{eq-thm-Local-average-generalization}) seem to be better for more accurate models, as evidenced in SwinTransformer and ResNet V2.

\begin{table}[tp]
\footnotesize{
\centering
\caption{Upper bounds on the true error (i.e., $\Pr(\vh(\vx) \ne y)$) of 20 DNN models which were pretrained on ImageNet dataset. Each bound for pretrained model $\vh$ was computed with $\delta=0.01$. The second column presents the test accuracy at top 1, as reported by Pytorch. Bold numbers are the best, while italic numbers are the second best for each model. Note that the first two bounds are vacuous, while  bounds (\ref{eq-thm-Local-Sensitivity-generalization},\ref{eq-thm-Local-average-generalization}) are non-vacuous for all cases.}
\begin{tabular}{|l|c|c|c|c|c|c|}
\hline
\textbf{Model} & \textbf{Acc@1} &  \textbf{Bound (\ref{eq-thm-Alg-robust-gen-Kawa22})} ($\downarrow$) & \textbf{Bound (\ref{eq-thm-Local-Robustness-generalization})} ($\downarrow$) & \textbf{Bound (\ref{eq-thm-Local-Sensitivity-generalization})} ($\downarrow$) & \textbf{Bound (\ref{eq-thm-Local-average-generalization})} ($\downarrow$) \\
\hline
ResNet18 V1             & 69.758      & 1.527\tiny{$\pm$0.005} & 1.527\tiny{$\pm$0.005} & \textit{0.917\tiny{$\pm$0.005}} & \textbf{0.625\tiny{$\pm$0.005}} \\
ResNet34 V1             & 73.314      & 1.462\tiny{$\pm$0.005} & 1.462\tiny{$\pm$0.005} & \textit{0.805\tiny{$\pm$0.004}} & \textbf{0.578\tiny{$\pm$0.004}} \\
ResNet50 V1             & 76.130      & 1.431\tiny{$\pm$0.004} & 1.430\tiny{$\pm$0.004} & \textit{0.743\tiny{$\pm$0.004}} & \textbf{0.546\tiny{$\pm$0.004}} \\
ResNet101 V1            & 77.374      & 1.401\tiny{$\pm$0.005} & 1.400\tiny{$\pm$0.005} & \textit{0.688\tiny{$\pm$0.005}} & \textbf{0.528\tiny{$\pm$0.004}} \\
ResNet152 V1            & 78.312      & 1.395\tiny{$\pm$0.004} & 1.394\tiny{$\pm$0.004} & \textit{0.673\tiny{$\pm$0.004}} & \textbf{0.515\tiny{$\pm$0.004}} \\
ResNet50 V2             & 80.858      & 1.379\tiny{$\pm$0.004} & 1.377\tiny{$\pm$0.005} & \textit{0.633\tiny{$\pm$0.004}} & \textbf{0.491\tiny{$\pm$0.004}} \\
ResNet101 V2            & 81.886      & 1.346\tiny{$\pm$0.004} & 1.344\tiny{$\pm$0.004} & \textit{0.571\tiny{$\pm$0.004}} & \textbf{0.474\tiny{$\pm$0.004}} \\
ResNet152 V2            & 82.284      & 1.337\tiny{$\pm$0.004} & 1.333\tiny{$\pm$0.004} & \textit{0.552\tiny{$\pm$0.004}} & \textbf{0.468\tiny{$\pm$0.004}} \\
SwinTransformer B       & 83.582      & 1.347\tiny{$\pm$0.004} & 1.345\tiny{$\pm$0.004} & \textit{0.563\tiny{$\pm$0.004}} & \textbf{0.456\tiny{$\pm$0.004}} \\
SwinTransformer T       & 81.474      & 1.389\tiny{$\pm$0.004} & 1.387\tiny{$\pm$0.004} & \textit{0.647\tiny{$\pm$0.004}} & \textbf{0.487\tiny{$\pm$0.004}} \\
SwinTransformer B V2  & 84.112      & 1.345\tiny{$\pm$0.004} & 1.342\tiny{$\pm$0.004} & \textit{0.551\tiny{$\pm$0.004}} & \textbf{0.444\tiny{$\pm$0.004}}  \\
SwinTransformer T V2  & 82.072      & 1.373\tiny{$\pm$0.004} & 1.372\tiny{$\pm$0.004} & \textit{0.613\tiny{$\pm$0.004}} & \textbf{0.472\tiny{$\pm$0.004}} \\
VGG13                   & 69.928      & 1.500\tiny{$\pm$0.005} & 1.499\tiny{$\pm$0.005} & \textit{0.879\tiny{$\pm$0.005}} & \textbf{0.625\tiny{$\pm$0.005}}  \\
VGG13 BN                & 71.586      & 1.504\tiny{$\pm$0.004} & 1.503\tiny{$\pm$0.005} & \textit{0.876\tiny{$\pm$0.004}} & \textbf{0.606\tiny{$\pm$0.004}} \\
VGG19                   & 72.376      & 1.470\tiny{$\pm$0.005} & 1.469\tiny{$\pm$0.005} & \textit{0.821\tiny{$\pm$0.005}} & \textbf{0.591\tiny{$\pm$0.005}} \\
VGG19 BN                & 74.218      & 1.464\tiny{$\pm$0.004} & 1.463\tiny{$\pm$0.005} & \textit{0.803\tiny{$\pm$0.004}} & \textbf{0.570\tiny{$\pm$0.004}}  \\
DenseNet121             & 74.434      & 1.457\tiny{$\pm$0.005} & 1.457\tiny{$\pm$0.005} & \textit{0.785\tiny{$\pm$0.005}} & \textbf{0.552\tiny{$\pm$0.005}}  \\
DenseNet161             & 77.138      & 1.400\tiny{$\pm$0.004} & 1.398\tiny{$\pm$0.004} & \textit{0.681\tiny{$\pm$0.004}} & \textbf{0.518\tiny{$\pm$0.004}}  \\
DenseNet169             & 75.600      & 1.422\tiny{$\pm$0.004} & 1.421\tiny{$\pm$0.004} & \textit{0.725\tiny{$\pm$0.004}} & \textbf{0.537\tiny{$\pm$0.004}} \\
DenseNet201             & 76.896      & 1.393\tiny{$\pm$0.004} & 1.392\tiny{$\pm$0.004} & \textit{0.673\tiny{$\pm$0.005}} & \textbf{0.522\tiny{$\pm$0.005}} \\
\hline
\multicolumn{2}{|l|}{\textbf{Correlation to Acc@1}} & -0.967 & -0.967 & -0.979 & -0.993 \\ \hline
\end{tabular}
\label{tab:imagenet-bounds-comparison}}
\end{table}


Ignoring the uncertainty term, we next consider how well those bounds can estimate the true error of a model. In this case, we focus on their main quantities: 
\begin{eqnarray}
Rob &=& F(\mS, \vh) + \epsilon(\mS) \\ 
LocalRob &=& F(\mS, \vh) + {  \sum_{i \in \mT_S}  \frac{ n_i}{n}  \epsilon_i(\vh) } \\ 
LocalSen &=& F(\mS, \vh) + {  \sum_{i \in \mT_S}  \frac{ n_i}{n}  \bar{\epsilon}_i(\vh) } \\
LocalAvg &=& \sum_{i \in \mT_S}  \frac{ n_i}{n}  a_i(\vh) 
\end{eqnarray}
$Rob$ comes from prior works, while $LocalRob, LocalSen$ and $LocalAvg$ come from our bounds. Table~\ref{tab:imagenet-uncertainty-comparison} summarizes the results. It tells us that $LocalSen$ and $LocalAvg$ are excellent estimates for the true error. Meanwhile both $Rob$ and $LocalRob$ are bad estimators. This behaviors also appear in other 10 models. To the best of our knowledge, $LocalSen$ and $LocalAvg$ provides the best estimates in the literature on this large-scale setting, without any modification for publicly pretrained models.

\begin{table}[tp]
\centering
\footnotesize
\caption{Estimates for the true error (i.e., $\Pr(\vh(\vx) \ne y)$) of 20 models on ImageNet, ignoring the uncertainty term. Bold numbers are the best, while italic numbers are the second best for each model.} 
\begin{tabular}{|l|c|c|c|c|c|c|}
\hline
\textbf{Model} & \textbf{Acc@1} &  \textbf{Rob} ($\downarrow$) & \textbf{LocalRob} ($\downarrow$) & \textbf{LocalSen} ($\downarrow$) & \textbf{LocalAvg} ($\downarrow$) \\
\hline
ResNet18 V1          & 69.758 & 1.212\tiny{$\pm$4.0e-5} & 1.212\tiny{$\pm$4.0e-5} & \textit{0.602\tiny{$\pm$1.9e-4}} & \textbf{0.310\tiny{$\pm$3.2e-4}}  \\
ResNet34 V1          & 73.314 & 1.157\tiny{$\pm$4.0e-5} & 1.156\tiny{$\pm$6.0e-5} & \textit{0.499\tiny{$\pm$1.8e-4}} & \textbf{0.272\tiny{$\pm$2.6e-4}} \\
ResNet50 V1          & 76.130 & 1.131\tiny{$\pm$6.0e-5} & 1.130\tiny{$\pm$6.0e-5} & \textit{0.443\tiny{$\pm$3.0e-4}} & \textbf{0.246\tiny{$\pm$4.4e-4}} \\
ResNet101 V1         & 77.374 & 1.105\tiny{$\pm$5.0e-5} & 1.104\tiny{$\pm$1.2e-4} & \textit{0.392\tiny{$\pm$1.7e-4}} & \textbf{0.231\tiny{$\pm$2.3e-4}} \\
ResNet152 V1         & 78.312 & 1.101\tiny{$\pm$4.0e-5} & 1.100\tiny{$\pm$7.0e-5} & \textit{0.379\tiny{$\pm$2.4e-4}} & \textbf{0.221\tiny{$\pm$3.5e-4}} \\
ResNet50 V2          & 80.858 & 1.089\tiny{$\pm$4.0e-5} & 1.088\tiny{$\pm$3.1e-4} & \textit{0.344\tiny{$\pm$3.5e-4}} & \textbf{0.201\tiny{$\pm$3.6e-4}} \\
ResNet101 V2         & 81.886 & 1.060\tiny{$\pm$2.0e-5} & 1.057\tiny{$\pm$3.0e-4} & \textit{0.285\tiny{$\pm$3.0e-4}} & \textbf{0.188\tiny{$\pm$3.4e-4}} \\
ResNet152 V2         & 82.284 & 1.052\tiny{$\pm$4.0e-5} & 1.048\tiny{$\pm$2.6e-4} & \textit{0.267\tiny{$\pm$2.5e-4}} & \textbf{0.183\tiny{$\pm$2.9e-4}} \\
SwinTransformer B    & 83.582 & 1.065\tiny{$\pm$4.0e-5} & 1.062\tiny{$\pm$1.9e-4} & \textit{0.280\tiny{$\pm$1.0e-4}} & \textbf{0.173\tiny{$\pm$1.0e-4}}  \\
SwinTransformer T    & 81.474 & 1.100\tiny{$\pm$4.0e-5} & 1.098\tiny{$\pm$1.6e-4} & \textit{0.358\tiny{$\pm$3.3e-4}} & \textbf{0.198\tiny{$\pm$3.7e-4}} \\
SwinTransformer B V2  & 84.112 & 1.064\tiny{$\pm$2.0e-5} & 1.061\tiny{$\pm$2.1e-4} & \textit{0.270\tiny{$\pm$2.4e-4}} & \textbf{0.163\tiny{$\pm$2.5e-4}} \\
SwinTransformer T V2  & 82.072 & 1.087\tiny{$\pm$4.0e-5} & 1.086\tiny{$\pm$1.4e-4} & \textit{0.327\tiny{$\pm$3.8e-4}} & \textbf{0.186\tiny{$\pm$4.0e-4}}  \\
VGG13                & 69.928 & 1.184\tiny{$\pm$5.0e-5} & 1.184\tiny{$\pm$6.0e-5} & \textit{0.563\tiny{$\pm$2.6e-4}} & \textbf{0.310\tiny{$\pm$3.9e-4}}  \\
VGG13 BN             & 71.586 & 1.192\tiny{$\pm$4.0e-5} & 1.192\tiny{$\pm$5.0e-5} & \textit{0.564\tiny{$\pm$2.8e-4}} & \textbf{0.295\tiny{$\pm$5.1e-4}} \\
VGG19                & 72.376 & 1.161\tiny{$\pm$6.0e-5} & 1.161\tiny{$\pm$6.0e-5} & \textit{0.512\tiny{$\pm$3.3e-4}} & \textbf{0.282\tiny{$\pm$3.5e-4}} \\
VGG19 BN             & 74.218 & 1.159\tiny{$\pm$4.0e-5} & 1.159\tiny{$\pm$8.0e-5} & \textit{0.499\tiny{$\pm$2.9e-4}} & \textbf{0.265\tiny{$\pm$3.9e-4}} \\
DenseNet121          & 74.434 & 1.156\tiny{$\pm$4.0e-5} & 1.156\tiny{$\pm$6.0e-5} & \textit{0.484\tiny{$\pm$1.9e-4}} & \textbf{0.251\tiny{$\pm$3.5e-4}} \\
DenseNet161          & 77.138 & 1.105\tiny{$\pm$4.0e-5} & 1.104\tiny{$\pm$5.0e-5} & \textit{0.386\tiny{$\pm$1.9e-4}} & \textbf{0.224\tiny{$\pm$2.6e-4}} \\
DenseNet169          & 75.600 & 1.124\tiny{$\pm$4.0e-5} & 1.123\tiny{$\pm$6.0e-5} & \textit{0.427\tiny{$\pm$1.6e-4}} & \textbf{0.238\tiny{$\pm$2.7e-4}} \\
DenseNet201          & 76.896 & 1.098\tiny{$\pm$4.0e-5} & 1.097\tiny{$\pm$4.0e-5} & \textit{0.378\tiny{$\pm$2.1e-4}} & \textbf{0.227\tiny{$\pm$3.1e-4}} \\
\hline
\multicolumn{2}{|l|}{\textbf{Correlation to Acc@1}} & -0.956 & -0.956 & -0.977 & -0.993 \\ \hline
\end{tabular}
\label{tab:imagenet-uncertainty-comparison}
\end{table}

\subsection{Unsupervised learning with PCA on moderate-size datasets}

We next evaluate those bounds (\ref{eq-thm-Alg-robust-gen-Kawa22},\ref{eq-thm-Local-Robustness-generalization},\ref{eq-thm-Local-Sensitivity-generalization},\ref{eq-thm-Local-average-generalization}) on a different task, i.e., dimensionality reduction with PCA. Two public  datasets with moderate size are used: CIFAR10 and SVHN. The training loss for PCA can be seen in Example 2 of Appendix~\ref{subsec-Examples}, and is not positive. Therefore, any positive estimates are vacuous. The setup for this experiment appears in Appendix~\ref{app-Experimental-setup-results}. PCA was run using various number $d$ of principal components, ranging from 100 to 1000. 

Table \ref{tab:pca-result} reports the results. 
For both datasets, \textit{Rob} produces large positive values, indicating vacuousness. This suggests that prior robustness-based bounds fail to assess model's performance on unseen data. Such a failure happens for every choice of $d$ in our experiments. On the other hand, \textit{LocalRob} and \textit{LocalAvg} effectively estimate the true loss. Surprisingly, \textit{LocalRob} in this evaluation is non-vacuous. Our more evaluations for classifiers trained on these datasets also indicate non-vacuousness of \textit{LocalRob} (see Appendix~\ref{app-Experimental-setup-results}). \textit{LocalAvg} often provides the best estimate.
There is also a strong correlation between LocalAvg and valid loss of PCA. It suggests that our bounds can better reflect PCA's generalization ability than prior ones.

\begin{table}[tp]
\footnotesize
\centering
\caption{Valid loss and other measures for PCA as the number $d$ of components varies. CIFAR10 and SVHN datasets are used in this experiment.}
\begin{tabular}{|c|c|c|c|c|c|c|}
\hline
\textbf{Dataset} &$d$ & \textbf{Valid loss} ($\downarrow$) & \textbf{Rob} ($\downarrow$) & \textbf{LocalRob} ($\downarrow$) & \textbf{LocalSen} ($\downarrow$ & \textbf{LocalAvg} ($\downarrow$) \\
\hline \hline
\multirow{4}{*}{CIFAR10}&100 & -863.91 & 2126.54 & -615.86 & \textit{-774.04} & \textbf{-859.53}\\
&300 & -876.20 & 2113.71 & -617.56 & \textit{-781.93} & \textbf{-871.37}\\
&500 & -879.59 & 2109.51 & -618.13 & \textit{-784.06} & \textbf{-874.64}\\
&1000 & -882.07 & 2106.31 & -618.62 & \textit{-785.62} &\textbf{-877.04}\\
\hline \hline
\multirow{4}{*}{SVHN}& 100 & -745.34 & 2134.03 & -623.59 & \textit{-693.52} & \textbf{-741.66}\\
&300 & -747.53 & 2132.25 & -623.46 & \textit{-694.89} & \textbf{-743.59}\\
&500 & -747.83 & 2131.99 & -623.40 & \textit{-695.05} & \textbf{-743.86}\\
&1000 & -747.98 & 2131.86 & -623.38 & \textit{-695.13} & \textbf{-743.99}\\
\hline
\end{tabular}
\label{tab:pca-result}
\end{table}

\section{Conclusion}\label{sec-Conclusion}

We carefully review prior work on robustness-based generalization bounds and identify their vacuousness. We then develop tighter  bounds by leveraging the local behaviors of a model at different small areas of the input space. Except i.i.d, our bounds require no assumption and avoid some serious limitations of prior bounds. For example, for a classification problem with overlapping classes, prior bounds are always vacuous for the best classifier, while our bounds are guaranteed to converge to the true error of that classifier as the training size increases.

Therefore, our new bounds provide effective tools for model selection and comparison. Interestingly, two of our bounds are empirically non-vacuous for a large class of  publicly pretrained deep neural networks. This would motivate future development of new theories that answer the biggest open challenge in deep learning  \cite{zhang2017DNNgeneralization}, i.e. explaining the high generalization ability of deep neural networks.

There remain some limitations in this work. First, the non-vacuousness of  our bounds is only empirical. We used an external dataset to approximate the robustness/sensitivity quantities, which may not reflect well their true values. Removing such empirical approximations requires extensive studies.  Second, computing our bounds may require using the training set, which can be costly for large datasets.

A number of directions can be developed from our work. First, one can use our bounds to analyze the connection between adversarial robustness \cite{qin2019adversarial} and generalization. We hypothesize that a model that is not adversarially robust at few small areas still generalizes well. This is partly evidenced in  Table~\ref{tab:imagenet-uncertainty-comparison}, in which robustness levels (\textit{Rob} and \textit{LocalRob}) are vacuous in all cases, suggesting that those models are highly prone to adversarial attacks. Another evidence can be observed from Example~3 of Appendix~\ref{subsec-Examples}. Meanwhile, \textit{LocalSen} and \textit{LocalAvg} are non-vacuous and often match with the test error, sugesting that those models generalize well in the classical sense. Second, one can improve robustness-based bounds further by improving the uncertainty term. \citet{kawaguchi2022robustness} made significant progress in this direction, but our  estimates reported in Table~\ref{tab:additional-result-2} and Table~\ref{tab:additional-result-imagenet-uncertainty}  in Appendix~\ref{app-Experimental-setup-results} suggests that it is not enough.

\backmatter



\pagebreak
\begin{appendices}

\input{RobGen-sup}




\end{appendices}


\bibliography{complexity,ann,other}


\end{document}

%% file: math_commands.tex
\usepackage{amsmath,amsfonts,bm}

\def\eqref#1{equation~\ref{#1}}

\def\1{\bm{1}}

\def\eps{{\epsilon}}

\def\vf{{\bm{f}}}

\def\vh{{\bm{h}}}

\def\vs{{\bm{s}}}

\def\vw{{\bm{w}}}
\def\vx{{\bm{x}}}

\def\vz{{\bm{z}}}

\def\mD{{\bm{D}}}

\def\mS{{\bm{S}}}
\def\mT{{\bm{T}}}

\DeclareMathAlphabet{\mathsfit}{\encodingdefault}{\sfdefault}{m}{sl}
\SetMathAlphabet{\mathsfit}{bold}{\encodingdefault}{\sfdefault}{bx}{n}

\def\gA{{\mathcal{A}}}

\def\gH{{\mathcal{H}}}

\def\gN{{\mathcal{N}}}
\def\gO{{\mathcal{O}}}

\def\gV{{\mathcal{V}}}

\def\gX{{\mathcal{X}}}

\def\gZ{{\mathcal{Z}}}

\def\sN{{\mathbb{N}}}

\newcommand{\E}{\mathbb{E}}

\newcommand{\R}{\mathbb{R}}



%% file: RobGen-sup.tex

\section{Proofs of the vacuousness}\label{app-Proof-vacuousness}

\begin{proof}[Proof of Theorem \ref{thm-Bayes-classifier}]
Consider the partition $\Gamma$ that decomposes $\gZ$ into two parts $\gZ_1 = \gO$ and $\gZ_2 = \gZ \setminus \gO$. We can decompose the expected loss as:
\[ F(P, \vh^*) = P(\gZ_1) a_1(\vh^*) + P(\gZ_2) a_2(\vh^*) \]
Note that $a_2(\vh^*) =0$ since $\vh^*$ can make accurate prediction for any example in $\gZ_2$, while $a_1(\vh^*) \le 1$. Therefore $F(P, \vh^*) \le P(\gZ_1)$.

Because $\gO$ has nonzero measure and contains only examples with zero margin, there exist two examples $\vs_1$ and $\vs_2$ in $\gO$ which are arbitrarily close to each other but have different labels. Then $| \ell(\vh^*, \vs_1) -\ell(\vh^*, \vs_2) | = 1$. This immediately implies $\epsilon_o(\vh^*, \gO) =1$.

Next we consider partition $\Gamma_o$. Since it decomposes $\gO$ into finite number of subsets, there exists a subset $\gO_k$ with nonzero measure. Using the same arguments as before, we can show $\epsilon_o(\vh^*, \gO_k) =1$,  completing the proof.
\end{proof}

\begin{proof}[Proof of Lemma \ref{lem-Bayes-classifier-local}]
Consider the partition $\Gamma$ that decomposes $\gZ$ into two parts $\gO$ and $\gZ \setminus \gO$. Denote $n_{-o} = | \mS \cap (\gZ \setminus \gO)|, p_o = P(\gO), p_{-o} = P(\gZ \setminus \gO)$. Since $\mS$ contains i.i.d. samples from distribution $P$, $(n_o, n_{-o})$ is a multinomial random variable with parameters $n$ and $(p_o, p_{-o})$. As a result, Lemma 4 in \cite{kawaguchi2022robustness} shows that the following holds with probability at least $1- \delta$:
\begin{equation}
\label{lem-Bayes-classifier-local-eq-01}
\frac{n_o}{n} \le p_o + 
\begin{cases}
   \sqrt{p_o\frac{\ln(2/\delta)}{n} }, \text{ if } p_o > \frac{\ln(2/\delta)}{4n} \\
   \frac{2\ln(2/\delta)}{n}, \text{ if } p_o \le \frac{\ln(2/\delta)}{4n} 
\end{cases}
\end{equation}
Since $\delta \ge 2e^{-n/4}$ implies  $n \ge 4\ln(2/\delta)$, it is easy to see that $\frac{2\ln(2/\delta)}{n} \le \sqrt{\frac{\ln(2/\delta)}{n}}$ and $\sqrt{p_o\frac{\ln(2/\delta)}{n} } \le \sqrt{\frac{\ln(2/\delta)}{n} }$. As a result, we have 
\begin{equation}
\label{lem-Bayes-classifier-local-eq-01}
\Pr\left(\frac{n_o}{n} \le p_o + \sqrt{\frac{\ln(2/\delta)}{n} }\right) \ge 1- \delta 
\end{equation}
which completes the proof.
\end{proof}

\section{Proofs for main results} \label{app-Proofs-for-main-results}

In order to present the proofs of the main results, we need the following observation.

\begin{lemma}\label{app-lem-error-decomposition}
Consider a model $\vh$ and a dataset $\mS$ with $n$ i.i.d. samples from distribution $P$. Denote $P(\gZ_i)$ as the probability that a random sample $\vz \sim P$ belongs to area $\gZ_i$. Then: 
 \begin{equation}
 \label{eq-app-lem-error-decomposition}
   F(P, \vh) = F(\mS, \vh) +   \sum_{i=1}^{K } a_i(\vh) \left[P(\gZ_i) -  \frac{n_i}{n} \right] +  \sum_{i \in \mT_S}  \frac{ n_i}{n} \left[ a_i(\vh) - F(\mS_i, \vh) \right] 
\end{equation}
\end{lemma}

\begin{proof}
Firstly, we  make the following  decomposition:
\begin{eqnarray}
\nonumber
  F(P, \vh) &=&  F(P, \vh) - \sum_{i=1}^{K } \frac{n_i}{n} \E_{\vz \sim  P}[\ell(\vh,\vz) | \vz \in \gZ_i]  \\
\label{eq-app-lem-error-decomposition-01} 
& & +  \sum_{i=1}^{K } \frac{n_i}{n} \E_{\vz \sim  P} [\ell(\vh,\vz) | \vz \in \gZ_i] - F(\mS, \vh) +  F(\mS, \vh) 
\end{eqnarray}

Secondly, observe that
\begin{eqnarray}
\nonumber
 F(P, \vh) - \sum_{i=1}^{K } \frac{n_i}{n} \E_{\vz \sim  P}[\ell(\vh,\vz) | \vz \in \gZ_i] 
  &=&  \sum_{i=1}^{K } P(\gZ_i) \E_{\vz \sim  P}[\ell(\vh,\vz) | \vz \in \gZ_i] \\
  \nonumber
  & & - \sum_{i=1}^{K } \frac{n_i}{n} \E_{\vz \sim  P}[\ell(\vh,\vz) | \vz \in \gZ_i]  \\
\nonumber
 &=&  \sum_{i=1}^{K } \E_{\vz \sim  P}[\ell(\vh,\vz) | \vz \in \gZ_i] \left[P(\gZ_i) -  \frac{n_i}{n} \right]   \\
 \label{eq-app-lem-error-decomposition-02} 
 &=&  \sum_{i=1}^{K } a_i(\vh) \left[P(\gZ_i) -  \frac{n_i}{n} \right] 
 \end{eqnarray}

Furthermore,
\begin{eqnarray}
\sum_{i=1}^{K } \frac{n_i}{n} \E_{\vz \sim  P}[\ell(\vh,\vz) | \vz \in \gZ_i] - F(\mS, \vh)
\nonumber
 &=& \sum_{i=1}^{K } \frac{n_i}{n} \E_{\vz \sim  P}[\ell(\vh,\vz) | \vz \in \gZ_i] - \frac{1}{n} \sum_{\vs \in \mS} \ell(\vh,\vs) \\
\nonumber
&=& \sum_{i \in \mT_S} \frac{n_i}{n} \E_{\vz \sim  P} [\ell(\vh,\vz) | \vz \in \gZ_i] - \frac{1}{n} \sum_{i \in \mT_S} \sum_{\vs \in \mS_i} \ell(\vh,\vs) \\
\nonumber
&=& \frac{1}{n} \sum_{i \in \mT_S} \left[ n_i \E_{\vz \sim  P} [\ell(\vh,\vz) | \vz \in \gZ_i] -  \sum_{\vs \in \mS_i} \ell(\vh,\vs) \right] \\
\nonumber
&=& \frac{1}{n} \sum_{i \in \mT_S} n_i \left[ \E_{\vz \sim  P} [\ell(\vh,\vz) | \vz \in \gZ_i] - \frac{1}{n_i} \sum_{\vs \in \mS_i} \ell(\vh,\vs) \right] \\
\label{eq-app-lem-error-decomposition-03} 
&=& \sum_{i \in \mT_S}  \frac{ n_i}{n} \left[ a_i(\vh) - F(\mS_i, \vh) \right] 
\end{eqnarray}
Combining the decomposition (\ref{eq-app-lem-error-decomposition-01}) with (\ref{eq-app-lem-error-decomposition-02}) and (\ref{eq-app-lem-error-decomposition-03}) completes the proof.
\end{proof}

\begin{proof}[Proof of Theorem \ref{thm-Local-Robustness-generalization}]
By Lemma \ref{app-lem-error-decomposition} we have:
 \begin{equation}
 \label{eq-app-thm-Local-Robustness-generalization-1}
   F(P, \vh) = F(\mS, \vh) +   \sum_{i=1}^{K } a_i(\vh) \left[P(\gZ_i) -  \frac{n_i}{n} \right] +  \sum_{i \in \mT_S}  \frac{ n_i}{n} \left[ a_i(\vh) - F(\mS_i, \vh) \right] 
\end{equation}

Observe that
\begin{eqnarray}
 \sum_{i \in \mT_S}  \frac{ n_i}{n} \left[ a_i(\vh) - F(\mS_i, \vh) \right]  
&=& \sum_{i \in \mT_S}  \frac{ 1}{n} \left[ n_i a_i(\vh) - \sum_{\vs \in \mS_i} \ell(\vh,\vs) \right]  \\
 &=& \frac{1}{n} \sum_{i \in \mT_S}  \sum_{\vs \in \mS_i} \left[  a_i(\vh) -  \ell(\vh,\vs) \right] \\ 
&\le& \frac{1}{n} \sum_{i \in \mT_S}  \sum_{\vs \in \mS_i} \sup_{\vz \in \gZ_i} | \ell(\vh,\vz) -  \ell(\vh,\vs) | \\
&\le& \frac{1}{n} \sum_{i \in \mT_S}  \sum_{\vs \in \mS_i}  \epsilon_i(\vh)   \\
\label{eq-app-thm-Local-Robustness-generalization-3}
&=& \sum_{i \in \mT_S}  \frac{ n_i}{n}  \epsilon_i(\vh)  
\end{eqnarray}

Note that ($n_1, ..., n_K$) is an i.i.d multinomial random variable with parameters $n$ and $(P(\gZ_1), ..., P(\gZ_K))$. Therefore, according to Theorem~3 in \cite{kawaguchi2022robustness},  for any $\delta>0$, we have the following with probability at least $1 - \delta$:
\begin{eqnarray}
 \sum_{i=1}^{K } a_i(\vh) \left[P(\gZ_i) -  \frac{n_i}{n} \right]   &\le& Q \sqrt{\frac{| \mT_S| \log(2 K /\delta)}{n}} + a_c \frac{ 2 | \mT_S| \log(2 K /\delta)}{n}
\end{eqnarray}
where $Q = a_t \sqrt{2} + a_c, a_t = \sup_{i \in \mT_S} a_i(\vh)$, and $a_c = \sup_{j \notin \mT_S} a_j(\vh)$. Note that $a_i(\vh) \le C$ for any index $i$. It suggests that $Q \le C(\sqrt{2} + 1)$ and $a_c \le C$. As a result, 
\begin{eqnarray}
\label{eq-app-thm-Local-Robustness-generalization-2}
\sum_{i=1}^{K } a_i(\vh) \left[P(\gZ_i) -  \frac{n_i}{n} \right]  &\le& C(\sqrt{2} + 1) \sqrt{\frac{| \mT_S| \log(2 K /\delta)}{n}} + \frac{ 2 C | \mT_S| \log(2 K /\delta)}{n}
\end{eqnarray}

Combining (\ref{eq-app-thm-Local-Robustness-generalization-1}) and (\ref{eq-app-thm-Local-Robustness-generalization-2}) and (\ref{eq-app-thm-Local-Robustness-generalization-3})  completes the proof.
\end{proof}

\begin{proof}[Proof of Theorem \ref{thm-Local-Sensitivity-generalization}]
By Lemma \ref{app-lem-error-decomposition} we have:
 \begin{equation}
 \label{eq-app-thm-Local-Sensitivity-generalization-1}
   F(P, \vh) = F(\mS, \vh) +   \sum_{i=1}^{K } a_i(\vh) \left[P(\gZ_i) -  \frac{n_i}{n} \right] +  \sum_{i \in \mT_S}  \frac{ n_i}{n} \left[ a_i(\vh) - F(\mS_i, \vh) \right] 
\end{equation}

Observe that
\begin{eqnarray}
 \sum_{i \in \mT_S}  \frac{ n_i}{n} \left[ a_i(\vh) - F(\mS_i, \vh) \right]  
&=& \sum_{i \in \mT_S}  \frac{ 1}{n} \left[ n_i a_i(\vh) - \sum_{\vs \in \mS_i} \ell(\vh,\vs) \right]  \\
 &=& \frac{1}{n} \sum_{i \in \mT_S}  \sum_{\vs \in \mS_i} \left[  a_i(\vh) -  \ell(\vh,\vs) \right] \\ 
 &=& \frac{1}{n} \sum_{i \in \mT_S}  \sum_{\vs \in \mS_i} \E_{\vz}[\ell(\vh,\vz) - \ell(\vh,\vs) : \vz \in \gZ_i] \\ 
 &\le& \frac{1}{n} \sum_{i \in \mT_S}  \sum_{\vs \in \mS_i} \E_{\vz \in \gZ_i} | \ell(\vh,\vz) -  \ell(\vh,\vs) | \\
\label{eq-app-thm-Local-Sensitivity-generalization-3}
&=& \sum_{i \in \mT_S}  \frac{ n_i}{n}  \bar{\epsilon}_i(\vh)  
\end{eqnarray}

For any $\delta>0$, by using the same argument with the proof of Theorem \ref{thm-Local-Robustness-generalization}, we have the following with probability at least $1 - \delta$:
\begin{eqnarray}
\label{eq-app-thm-Local-Sensitivity-generalization-2}
 \sum_{i=1}^{K } a_i(\vh) \left[P(\gZ_i) -  \frac{n_i}{n} \right]   &\le& C(\sqrt{2} + 1) \sqrt{\frac{| \mT_S| \log(2 K /\delta)}{n}} + \frac{ 2 C | \mT_S| \log(2 K /\delta)}{n}
\end{eqnarray}

Combining (\ref{eq-app-thm-Local-Sensitivity-generalization-1}) and (\ref{eq-app-thm-Local-Sensitivity-generalization-2}) and (\ref{eq-app-thm-Local-Sensitivity-generalization-3})   completes the proof.
\end{proof}

\begin{proof}[Proof of Theorem \ref{thm-Local-average-generalization}]
By Lemma \ref{app-lem-error-decomposition} we have:
 \begin{equation}
 \label{eq-app-thm-Local-average-generalization-1}
   F(P, \vh) = F(\mS, \vh) +   \sum_{i=1}^{K } a_i(\vh) \left[P(\gZ_i) -  \frac{n_i}{n} \right] +  \sum_{i \in \mT_S}  \frac{ n_i}{n} \left[ a_i(\vh) - F(\mS_i, \vh) \right] 
\end{equation}

Observe that
\begin{eqnarray}
\nonumber
F(\mS, \vh) + \sum_{i \in \mT_S}  \frac{ n_i}{n} \left[ a_i(\vh) - F(\mS_i, \vh) \right]  
&=& F(\mS, \vh) + \sum_{i \in \mT_S}  \frac{ n_i}{n} a_i(\vh)  -  \sum_{i \in \mT_S}  \frac{ n_i}{n} F(\mS_i, \vh) \\
&=& F(\mS, \vh) + \sum_{i \in \mT_S}  \frac{ n_i}{n} a_i(\vh)  - F(\mS, \vh)  \\
\label{eq-app-thm-Local-average-generalization-2}
&=& \sum_{i \in \mT_S}  \frac{ n_i}{n} a_i(\vh)
\end{eqnarray}

For any $\delta>0$, by using the same argument with the proof of Theorem \ref{thm-Local-Robustness-generalization}, we have the following with probability at least $1 - \delta$:
\begin{eqnarray}
\label{eq-app-thm-Local-average-generalization-3}
 \sum_{i=1}^{K } a_i(\vh) \left[P(\gZ_i) -  \frac{n_i}{n} \right]   &\le& C(\sqrt{2} + 1) \sqrt{\frac{| \mT_S| \log(2 K /\delta)}{n}} + \frac{ 2 C | \mT_S| \log(2 K /\delta)}{n}
\end{eqnarray}

Combining (\ref{eq-app-thm-Local-average-generalization-1}) and (\ref{eq-app-thm-Local-average-generalization-2}) and (\ref{eq-app-thm-Local-average-generalization-3})  completes the proof.
\end{proof}

\subsection{Proof of bound comparison} \label{app-proof-bound-comparison}
\begin{proof}[Proof of Lemma \ref{lem-bound-compare}]
By definitions of $\epsilon_i(\vh)$ and $\epsilon(\mS)$, we can see that $\epsilon_i(\vh) \le \epsilon(\mS)$ for any index $i \in [K]$. Therefore any convex combination of all $\epsilon_i(\vh)$'s should not exceed $\epsilon(\mS)$. As a result, $\sum_{i \in \mT_S}  \frac{ n_i}{n} \epsilon_i(\vh) \le  \eps(\mS)$.

Observe that $\bar{\epsilon}_i(\vh) = \frac{1}{n_i} \sum_{\vs \in \mS_i} \E_{\vz \in \gZ_i} | \ell(\vh,\vz) -  \ell(\vh,\vs) |  \le \frac{1}{n_i} \sum_{\vs \in \mS_i}  \sup_{\vz \in \gZ_i} | \ell(\vh, \vs) - \ell(\vh,\vz) | \le  \sup_{\vs \in \mS_i,  \vz \in \gZ_i} | \ell(\vh, \vs) - \ell(\vh,\vz) | = \epsilon_i(\vh)$. Therefore, $ \sum_{i \in \mT_S}  \frac{ n_i}{n}  \bar{\epsilon}_i(\vh)  \le  \sum_{i \in \mT_S}  \frac{ n_i}{n}  {\epsilon}_i(\vh) $.

Note further that 
\begin{eqnarray}
 \sum_{i \in \mT_S}  \frac{ n_i}{n} a_i(\vh) - F(\mS, \vh) 
 &=& \frac{ 1}{n}  \sum_{i \in \mT_S}  n_i a_i(\vh) - \frac{ 1}{n}  \sum_{i \in \mT_S} \sum_{\vs \in \mS_i} \ell(\vh,\vs)  \\
&=& \frac{ 1}{n}  \sum_{i \in \mT_S}  \left[ n_i a_i(\vh) - \sum_{\vs \in \mS_i} \ell(\vh,\vs) \right]  \\
 &=& \frac{1}{n} \sum_{i \in \mT_S}  \sum_{\vs \in \mS_i} \left[  a_i(\vh) -  \ell(\vh,\vs) \right] \\ 
 &=& \frac{1}{n} \sum_{i \in \mT_S}  \sum_{\vs \in \mS_i} \E_{\vz}[\ell(\vh,\vz) - \ell(\vh,\vs) : \vz \in \gZ_i] \\ 
 &\le& \frac{1}{n} \sum_{i \in \mT_S}  \sum_{\vs \in \mS_i} \E_{\vz \in \gZ_i} | \ell(\vh,\vz) -  \ell(\vh,\vs) | \\
&=& \sum_{i \in \mT_S}  \frac{ n_i}{n}  \bar{\epsilon}_i(\vh)  
\end{eqnarray}
This means $\sum_{i \in \mT_S}  \frac{ n_i}{n} a_i(\vh) \le F(\mS, \vh) +  \sum_{i \in \mT_S}  \frac{ n_i}{n}  \bar{\epsilon}_i(\vh)   $, completing the proof.
\end{proof}

\subsection{Proof of lower bound} \label{app-proof-lower-bound}
\begin{proof}[Proof of Theorem \ref{thm-Local-Rob-generalization-lower-bound}]
Denote $p_i = P(\gZ_i)$ for each index $i \in [K]$. We can decompose $E = \E_{\vh} [F(P, \vh)] = \E_{\vh} \left[\sum_{i=1}^K p_i {a}_i(\vh) \right] =  \sum_{i=1}^K p_i \E_{\vh}[{a}_i(\vh)] = \sum_{i=1}^K p_i \bar{a}_i$. Note that all $\bar{a}_i$'s are fixed w.r.t the sampling of $\mS$.

For   any $M \in (0, \beta/\hat{a})$, Lemma 2 in \cite{kawaguchi2022robustness} shows that
\begin{eqnarray}
\nonumber
\Pr\left(\sum_{i=1}^K p_i \bar{a}_i \ge \sum_{i=1}^K \frac{n_i}{n} \bar{a}_i - M  \right)  
&\ge& 1- \exp\left(-\frac{nM}{2\hat{a}} \min\left\{1, \frac{\hat{a}M}{\beta}\right\}\right) = 1- \exp\left(-\frac{nM^2}{2\beta}\right) 
\end{eqnarray}
For any $\delta> \exp\left(-\frac{n\beta}{2\hat{a}^2}\right)$, choosing $M = \sqrt{\frac{-2\beta \ln\delta}{n}}$, we obtain
\begin{eqnarray}
\Pr\left(\sum_{i=1}^K p_i \bar{a}_i \ge \sum_{i=1}^K \frac{n_i}{n} \bar{a}_i - \sqrt{\frac{-2\beta \ln\delta}{n}}  \right)  &\ge&  1- \delta
\end{eqnarray}
In other words, the following holds with probability at least $1-\delta$:
\begin{eqnarray}
\label{app-eq-lower-bound-01}
E &\ge& \sum_{i=1}^K \frac{n_i}{n} \bar{a}_i - \sqrt{\frac{-2\beta \ln\delta}{n}} =  \sum_{i \in \mT_S} \frac{n_i}{n} \bar{a}_i - \sqrt{\frac{-2\beta \ln\delta}{n}} 
\end{eqnarray}

We next observe that $\sqrt{\beta} = \sqrt{ 2\sum_{i=1}^K p_i \bar{a}_i^2} \le \sqrt{ 2\hat{a}\sum_{i=1}^K p_i \bar{a}_i} = \sqrt{ 2\hat{a}} \sqrt{ E}$. Utilizing this information into (\ref{app-eq-lower-bound-01}), we have  the following  with probability at least $1-\delta$:
\begin{eqnarray}
\label{app-eq-lower-bound-02}
E &\ge& \sum_{i \in \mT_S} \frac{n_i}{n} \bar{a}_i -  \sqrt{E} \sqrt{\frac{-4 \hat{a} \ln\delta}{n}}
\end{eqnarray}
Solving this inequality for $ \sqrt{E}$ will complete the proof.
\end{proof}

\begin{proof}[Proof of Theorem \ref{thm-Local-Rob-generalization-lower-bound-best}]
Denote $p_i = P(\gZ_i)$ for each index $i \in [K]$. Consider $E = F(P, \vh^*) =  \sum_{i=1}^K p_i a_i(\vh^*)$ which is fixed w.r.t to the sampling of $\mS$. Then we can use the same arguments as the proof before to obtain the required bound.
\end{proof}

\subsection{Concentration for multinomial random variables}

We restate Lemma~7 in \cite{kawaguchi2022robustness} as follows.

\begin{lemma}\label{lem-Concentration-multinomial}
Given any $a_i(\gZ) \ge 0, \forall i \le K$, denote $a_o = \max_{j \notin \mT_S} a_j(\gZ)$. Let $(n_1, ..., n_K)$ be a multinomial random vector with parameter $n$ and $(p_1, ..., p_K)$, meaning that $p_i = \Pr(n_i)$ and $n = \sum_{i=1}^{K} n_i$. For any $\delta>0$, the following holds with probablity at least $1-\delta$:
{\footnotesize \[
\sum_{i=1}^{K} a_i(\gZ) \left(p_i - \frac{n_i}{n}\right) 
 \le \sqrt{\frac{\ln(2K/\delta)}{n}} \left(\sum_{i \in \mT_S}[a_o + \sqrt{2} a_i(\gZ)] \sqrt{\frac{n_i}{n}} \right) + \frac{2\ln(2K/\delta)}{n} \left(a_o | \mT_S| +  \sum_{i \in \mT_S} a_i(\gZ) \right)
\]}
\end{lemma}

\section{More examples and comparison} \label{subsec-Examples}
We provide some more examples to compare our bounds with prior ones. We take some examples from \cite{xu2012robustnessGeneralize,kawaguchi2022robustness}.

\textbf{Example 1.} (Lipschitz continuous functions) A large class of models in practice has a common property that each member is often Lipschitz continuous in its input. One example is deep neural networks (DNN) with ReLU activations and bounded weights. When the loss $\ell$ is Lipschitz continuous, which is natural, then \cite{xu2012robustnessGeneralize,kawaguchi2022robustness} showed the robustness level $\epsilon(\mS)$. Specifically, if $\gZ$ is compact according to a metric $\rho$ and $\ell(\gA_S,\cdot)$ is Lipschitz continuous with Lipschitz constant $L$, i.e.,
\[ |\ell(\gA_S,\vz) - \ell(\gA_S,\vs) |  \le L \rho(\vz,\vs), \forall \vz, \vs \in \gZ \]
then algorithm $\gA$ is $(\gN(\gamma/2, \gZ, \rho), \gamma L)$-robust for any given $\gamma>0$, where $\gN(\gamma/2, \gZ, \rho)$ is the covering number of $\gZ$. It means that $\epsilon(\mS) = \gamma L$. The use of $L$ here makes the bound loose, since a significant change of $\ell$ in a small area will produce a large $L$. By using (\ref{eq-thm-Local-Robustness-generalization}), $\epsilon(\mS)$ is replaced by $ ub= \sum_{i \in \mT_S}  \frac{ n_i}{n}  \epsilon_i(\gA_S)$ for the  function $\gA_S$ trained on $\mS$. We show in Appendix~\ref{app-Proofs-for-examples} that $ub \le \gamma \sum_{i \in \mT_S}  \frac{ n_i}{n}  L_i$ where $L_i$ is the local Lipschitz constant of $\ell(\gA_S,\cdot)$ in area $\gZ_i$. 

Prior results require $L$ to depend on the whole family $\gH$, meaning $L$ is the maximum among the Lipschitz constants of all members of $\gH$. This fact suggests that $L$ should be unreasonably large. For example, $L$ can be \textit{exponential in the depth} of the neural network family \cite{arora2018strongerBounds,bartlett2017SpectralMarginDNN}. Example 7 in \cite{xu2012robustnessGeneralize} estimates the  Lipschitz constants of a DNN family and shows $L$ to be of order $\alpha^D$, where $D$ is the number of layers and $\alpha$ is the maximal norm of weight matrices. For common DNNs trained on real-life datasets, the norm of weight matrices is often greater than 1. This suggests that $L$ can be unreasonably large, e.g., of order $10^{40}$ for VGG-19 \cite{arora2018strongerBounds}. In contrast, our bound only depends on $L_i$'s of one specific model at some local areas where training points actually occur. \cite{khromov2024LipschitzNN} provided extensive evidences about small size and well-behaved distributions of local Lipschitz constants of many modern DNNs trained on real-life datasets. It suggests that $\sum_{i \in \mT_S}  \frac{ n_i}{n}  L_i$ is often significantly smaller than $L$ for modern DNNs.

\textbf{Example 2.} (Principal Component Analysis - PCA) Assume that each element in $\gZ$ has norm at most $B$. If we use the loss funtion $\ell(\{\vw_1, ...,\vw_d\}, \vz) = \sum_{j=1}^{d} (\vw_j^{\top} \vz)^2$, then PCA finds the first $d$ principle components by minimizing $-\sum_{\vs \in \mS} \ell(\{\vw_1, ...,\vw_d\}, \vs)$ with the constraint that $\| \vw_j \| =1$ and $\vw_j^{\top} \vw_i =0$ for $i \neq j$. According to \cite{xu2012robustnessGeneralize,kawaguchi2022robustness}, this algorithm is $(\gN(\gamma/2, \gZ, \| \cdot \|), 2d \gamma B)$-robust. Theorem~\ref{thm-Alg-robust-gen-Kawa22} suggests that the learned components $\vh^* = \{\vw_1^*, ...,\vw_d^*\}$ satisfies 
\begin{equation} \label{eq-example-PCA-Xu12}
F(P, \vh^*) \le  g_2(\gN, \mS, \delta) + F(\mS, \vh^*) + 2d \gamma B
\end{equation}

We show in Appendix~\ref{app-Proofs-for-examples} that:
\begin{equation} \label{eq-example-PCA-ours}
F(P, \vh^*) \le  g_2(\gN, \mS, \delta) + F(\mS, \vh^*) + 2d \gamma  \sum_{i \in \mT_S}  \frac{n_i}{n} B_i , 
\end{equation}
where $B_i$ is the maximal norm of any element in $\gZ_i$. Since $B_i \le B$ for any index $i$, we have $\sum_{i \in \mT_S}  \frac{n_i}{n} B_i \le B$. This implies that our bound for PCA is tighter than the prior ones. Note that directly using (\ref{eq-thm-Local-Sensitivity-generalization}) or (\ref{eq-thm-Local-average-generalization}) even leads to stronger bounds for PCA.

\textbf{Example 3.} (Discontinuity) Consider  an unknown function
\[y^*(x) = 
\begin{cases}
\mu & \text{ if } x \in [0,\nu] \\
f(x) & \text{ otherwise }
\end{cases}\]
where $\mu$ is a large constant, $f$ is a continuous function satisfying $| f(x) | \ll \mu$ for any $x \notin [0,\nu]$, and $\nu$ is a small positive constant of order $o(1/\mu^2)$. This function $y^*$ is continuous everywhere except two points. This function has a strange behavior in the interval $[0, \nu]$. In practice, some inherent sources may cause this behavior such as errors in measurement. 

We want to approximate $y^*$, based on a sample $\mS$. To do this, let's choose a family $\gH$ of continuous functions, and consider the learned member $h^*$ with a small $\E_{x} [\ell(h^*, (x,y^*(x)))]$. Because of having a small expected loss,  $h^*$ can well predict  $y^*$ almost everywhere. 

Prior bounds based on algorithmic robustness will be vacuous for any partition $\Gamma(\gX)$ with large areas, i.e. $2\nu = \max_{x,s \in \gX_i} | x -s |, \forall i \in [K]$. Indeed, there exist at most two areas  $\gX_j$ and $\gX_k$ that cover the interval $[0,\nu]$. This implies that robustness level of $h^*$ in $\gX_j$ or $\gX_k$ should be $\Theta(\mu)$, for the absolute loss, due to the fact that $y^*(x)$ has this property and that $h^*$ well approximates $y^*$. As a result, robustness level \textcolor{blue}{$\epsilon(\mS) = \Theta(\mu)$}  causes prior bounds (\ref{eq-thm-Alg-robust-gen-Xu12},\ref{eq-thm-Alg-robust-gen-Kawa22}) to be vacuous for large values of $\mu$. One way to improve is to choose partition $\Gamma(\gX)$ with very small areas, but at the cost to increase the uncertainty term $g_2$.  Those observations suggest that few outliers or errors in measurement can make those bounds trivial. This is a severe limitation.

Our bounds can avoid this limitation. Indeed, since interval $[0,\nu]$ is small, some training samples appear in this interval with a very small probability. We can see this fact for the case of uniform distribution over $\gX = [-B, B]$ for some constant $B > \nu$. In Appendix~\ref{app-Proofs-for-examples}, we point out that \textcolor{blue}{$ {  \sum_{i \in \mT_S}  \frac{ n_i}{n}  \epsilon_i(\vh^*) } = o(1/\mu)$} which is small, for large $n$. It means our bound is meaningful.

Although $y^*$ in this example seems non-natural, there are many real-life  problems where we need to find/approximate a discontinuous function. For example, solutions to hyperbolic partial differential equations, which describe a wide variety of conservative physical systems, can be discontinuous. In those cases, our bounds  exhibit significant advantages over prior ones.

\subsection{Proofs for some examples} \label{app-Proofs-for-examples}

\begin{proof}[Proof of Example 1]

By definition, $\epsilon_i(\gA_S) = \sup_{\vs \in \mS_i,  \vz \in \gZ_i} | \ell(\gA_S, \vs) - \ell(\gA_S,\vz) |$. Since $\ell(\gA_S, \vz)$ is $L_i$-Lipschitz continuous in $\vz$, we have $ | \ell(\gA_S, \vs) - \ell(\gA_S,\vz) | \le L_i \|   \vs - \vz \|$ for any $\vs \in \mS_i,  \vz \in \gZ_i$. Therefore, $\epsilon_i(\gA_S) \le \sup_{\vs \in \mS_i,  \vz \in \gZ_i}  L_i \|   \vs - \vz \| \le \gamma L_i$.
\end{proof}

\begin{proof}[Proof of Example 2]

Let  $\vh^* = \{\vw_1^*, ...,\vw_d^*\}$ be the solution  of PCA, learned from a given dataset $\mS$. For any $\vz, \vs \in \gZ_i$, observe that
\begin{eqnarray}
\nonumber
| \ell(\vh^*, \vz) - \ell(\vh^*, \vs) | 
&=& \left| \sum_{j=1}^{d} ({\vw_j^*}^{\top} \vz)^2 - \sum_{j=1}^{d} ({\vw_j^*}^{\top} \vs)^2 \right|  =  \left| \sum_{j=1}^{d} [({\vw_j^*}^{\top} \vz)^2 - ({\vw_j^*}^{\top} \vs)^2 ] \right| \\
&\le& \sum_{j=1}^{d} | [{\vw_j^*}^{\top} \vz - {\vw_j^*}^{\top} \vs] \cdot [{\vw_j^*}^{\top} \vz + {\vw_j^*}^{\top} \vs] | \\
&\le& \sum_{j=1}^{d} | {\vw_j^*}^{\top} \vz - {\vw_j^*}^{\top} \vs | \cdot |{\vw_j^*}^{\top} \vz + {\vw_j^*}^{\top} \vs | \\
&\le& \sum_{j=1}^{d}  \|\vw_j^*\|  \cdot \|\vz - \vs\|  \cdot \|\vw_j^*\|  \cdot \|\vz + \vs\| \\
&\le& 2 d \gamma B_i
\end{eqnarray}
where we have used the fact that $\|\vw_j^*\| = 1$, $\|\vz\| \le B_i$ and $\|\vs\| \le B_i$. As a result $ \epsilon_i(\vh^*)  \le 2 d \gamma B_i$ for any index $i \in [K]$. Combining this with Theorem~\ref{thm-Local-Robustness-generalization} completes the proof.
\end{proof}

\begin{proof}[Proof of Example 3]
Although $\epsilon_j(\vh^*)$ or $\epsilon_k(\vh^*)$ may be large, their role in the bound (\ref{eq-thm-Local-Robustness-generalization}) should be small. The reason is that $\frac{ n_j}{n}  \epsilon_j(\vh^*) + \frac{ n_k}{n}  \epsilon_k(\vh^*) \le \frac{ n_i + n_k}{n} \max\{\epsilon_j(\vh^*), \epsilon_k(\vh^*)\} =  \frac{ n_i + n_k}{n} \Theta(\mu)$. Note that $\frac{ n_j}{n}  + \frac{ n_k}{n} \xrightarrow{n \rightarrow \infty}  \frac{2B}{K} + \frac{2B}{K} = 4\nu$. Hence $\frac{ n_j}{n}  \epsilon_j(\vh^*) + \frac{ n_k}{n}  \epsilon_k(\vh^*)  \approx 4\nu\Theta(\mu) \approx o(1/\mu)$ which is small for sufficiently large $n$. Furthermore, $\epsilon_i(\vh^*) \approx 0$ for any $i \notin \{j,k\}$. As a result $ {  \sum_{i \in \mT_S}  \frac{ n_i}{n}  \epsilon_i(\vh^*) } = o(1/\mu)$.
\end{proof}

\section{More experimental  results}\label{app-Experimental-setup-results}

In this section, we provide more evaluations on the robustness-based bounds. These evaluations require us to train a model from scratch. They complement our large-scale evaluation before for publicly pretrained models on ImageNet.

\subsection{Setup for PCA}

The CIFAR10 and SVHN dataset is used in our experiments. There are 50,000 images of CIFAR10 are used for training and 10,000 images are used for validation, while SVHN has 73,257 images for training and 26,032 images for validation. To compute the measures, we divide input data space into 10000 disjoint partitions, meaning $K = 10000$ in all settings. Each partition has the centroid which is an input sample in the valid set. 
For PCA, we utilized the implementation by scikit-learn, using default settings. We varied the number $d$ of principal components to evaluate our and prior bounds.

\subsection{Classification task on moderate-size datasets}

\subsubsection{Setup}

For the classification task, we conducted our experiment using the ResNet and ShuffleNet implementations available in PyTorch 2.0.0. To ensure a fair comparison, we maintained default hyper-parameter settings across all models. Similarly, for optimization, we utilized the SGD implementation provided by PyTorch, using their default settings. We evaluated our models on the CIFAR10 dataset, preprocessing the images by converting their pixel values to torch Tensors and normalizing them to the standard range of [0, 1].

During training, we ran each model for 200 epochs, saving the model at the end of each epoch. The best model was determined based on its validation accuracy. These models were then used to calculate some measures. To provide a comprehensive review, we varied the number $K$ of areas of a partition with four values \{100, 500, 1000, 10000\}. Each setting was run five times to obtain  better estimates for the measures and accuracy. All experiments were conducted on an NVIDIA P100 GPU using PyTorch 2.0.0. 

\subsubsection{Results} \label{subsec:Classification-task}

Table \ref{tab:classification_comparison} presents some statistics about  the trained models. We observe that \textit{Rob} values are almost the same for all models, while the accuracy of those models differs. \textit{Rob} is vacuous in all cases, and cannot reflect well the performance of a model. In contrast, \textit{LocalRob} seems to be better. It can decrease as $K$ increases, which reflect well our analysis before. However, for some small $K$, \textit{LocalRob} can be large and far from the true error of a model.

\textit{LocalSen} and \textit{LocalAvg} are often small in all cases. Those measures are quite stable w.r.t different partitions of the input space. One can easily observe that those measures correlate well with the accuracy. A model with better accuracy often has smaller \textit{LocalSen} and \textit{LocalAvg}. This is very beneficial in practice. Note that \textit{LocalSen} seems to be better than \textit{LocalAvg} when approximating the true error of a model. \textit{LocalAvg} tends to underestimate the true error.

\begin{table}[tp]
\tiny
\centering
\caption{Estimates for the true errors of different models trained on CIFAR10 dataset, when the size $K$ of the partition $\Gamma$ varies.}
\begin{tabular}{|c|c|c|c|c|c|c|}
\hline
\multirow{2}{*}{\textbf{Model}} & \multirow{2}{*}{\textbf{Valid Acc}} & \multirow{2}{*}{\textbf{K}} & \multirow{2}{*}{\textbf{Rob}} & \multirow{2}{*}{\textbf{LocalRob}} & \multirow{2}{*}{\textbf{LocalSen}} & \multirow{2}{*}{\textbf{LocalAvg}} \\
& & & & & &\\
\hline \hline

\multirow{4}{*}{ResNet18} & \multirow{4}{*}{94.22 $\pm$ 7.59e-3}  
        & 100 & 1.00 $\pm$ 0.00 & 1.00 $\pm$ 2.10e-05 & 0.06 $\pm$ 5.02e-06 & 0.01 $\pm$ 2.02e-07 \\
&  & 500 & 1.00 $\pm$ 0.00 & 0.94 $\pm$ 1.78e-04 & 0.06 $\pm$ 8.42e-06 & 0.01 $\pm$ 1.89e-07 \\
&  & 1000 & 1.00 $\pm$ 0.00 & 0.87 $\pm$ 7.22e-05 & 0.06 $\pm$ 4.67e-06 & 0.01 $\pm$ 8.95e-08 \\
&  & 10000 & 1.00 $\pm$ 0.00 & 0.28 $\pm$ 3.29e-04 & 0.05 $\pm$ 6.77e-06 & 0.01 $\pm$ 1.55e-07 \\
\hline \hline

\multirow{4}{*}{ResNet34} & \multirow{4}{*}{94.26 $\pm$ 7.35e-03} 
        & 100 & 1.00 $\pm$ 0.00 & 0.99 $\pm$ 1.77e-05 & 0.06 $\pm$ 2.75e-06 & 0.01 $\pm$ 1.14e-07 \\
&  & 500 & 1.00 $\pm$ 0.00 & 0.94 $\pm$ 1.21e-05 & 0.06 $\pm$ 1.30e-06 & 0.01 $\pm$ 4.36e-08 \\
&  & 1000 & 1.00 $\pm$ 0.00 & 0.87 $\pm$ 9.08e-05 & 0.06 $\pm$ 3.54e-06 & 0.01 $\pm$ 6.24e-08 \\
&  & 10000 & 1.00 $\pm$ 0.00 & 0.28 $\pm$ 7.68e-05 & 0.06 $\pm$ 1.68e-06 & 0.01 $\pm$ 2.46e-08 \\
\hline \hline

\multirow{4}{*}{ResNet50} & \multirow{4}{*}{94.10 $\pm$ 3.64e-02} 
        & 100 & 1.00 $\pm$ 0.00 & 0.99 $\pm$ 1.42e-05 & 0.06 $\pm$ 8.06e-06 & 0.01 $\pm$ 5.07e-07 \\
&  & 500 & 1.00 $\pm$ 0.00 & 0.94 $\pm$ 3.23e-05 & 0.06 $\pm$ 1.30e-05 & 0.01 $\pm$ 5.21e-07 \\
&  & 1000 & 1.00 $\pm$ 0.00 & 0.87 $\pm$ 3.34e-04 & 0.06 $\pm$ 1.19e-05 & 0.01 $\pm$ 5.09e-07 \\
&  & 10000 & 1.00 $\pm$ 0.00 & 0.29 $\pm$ 1.30e-04 & 0.06 $\pm$ 1.65e-06 & 0.01 $\pm$ 1.14e-07 \\
\hline \hline

\multirow{4}{*}{\begin{tabular}{@{}c@{}} ShuffleNet \\ (V2\_X1\_0))\end{tabular}} & \multirow{4}{*}{92.02 $\pm$ 1.92e-02} 
        & 100 & 1.00 $\pm$ 0.00 & 1.00 $\pm$ 1.71e-06 & 0.08 $\pm$ 9.62e-06 & 0.02 $\pm$ 1.86e-07 \\
&  & 500 & 1.00 $\pm$ 0.00 & 0.96 $\pm$ 6.23e-05 & 0.08 $\pm$ 2.03e-06 & 0.02 $\pm$ 2.55e-07 \\
&  & 1000 & 1.00 $\pm$ 0.00 & 0.91 $\pm$ 5.52e-05 & 0.08 $\pm$ 2.57e-06 & 0.02 $\pm$ 1.85e-07\\
&  & 10000 & 1.00 $\pm$ 0.00 & 0.36 $\pm$ 2.40e-04 & 0.08 $\pm$ 2.73e-06 & 0.01 $\pm$ 1.72e-07 \\
\hline \hline

\multirow{4}{*}{\begin{tabular}{@{}c@{}}ShuffleNet \\ (V2\_X1\_5) \end{tabular}} & \multirow{4}{*}{92.45 $\pm$ 1.36e-02} 
        & 100 & 1.00 $\pm$ 0.00 & 1.00 $\pm$ 1.69e-06 & 0.08 $\pm$ 4.12e-06 & 0.02 $\pm$ 2.36e-07\\
&  & 500 & 1.00 $\pm$ 0.00 & 0.96 $\pm$ 5.77e-05 & 0.08 $\pm$ 8.49e-07 & 0.02 $\pm$ 2.43e-07 \\
&  & 1000 & 1.00 $\pm$ 0.00 & 0.90 $\pm$ 1.08e-04 & 0.08 $\pm$ 2.12e-06 & 0.02 $\pm$ 1.88e-07 \\
&  & 10000 & 1.00 $\pm$ 0.00 & 0.35 $\pm$ 4.95e-05 & 0.08 $\pm$ 1.84e-06 & 0.01 $\pm$ 5.64e-08 \\
\hline \hline

\multirow{4}{*}{\begin{tabular}{@{}c@{}}ShuffleNet \\ (V2\_X2\_0)\end{tabular}} & \multirow{4}{*}{92.65 $\pm$ 3.46e-02} 
        & 100 & 1.00 $\pm$ 0.00 & 1.00 $\pm$ 7.13e-06 & 0.08 $\pm$ 8.98e-06 & 0.02 $\pm$ 3.98e-07 \\
&  & 500 & 1.00 $\pm$ 0.00 & 0.95 $\pm$ 4.06e-05 & 0.08 $\pm$ 1.08e-05 & 0.02 $\pm$ 3.34e-07 \\
&  & 1000 & 1.00 $\pm$ 0.00 & 0.89 $\pm$ 5.42e-04 & 0.08 $\pm$ 1.60e-05 & 0.02 $\pm$ 4.49e-07 \\
&  & 10000 & 1.00 $\pm$ 0.00 & 0.35 $\pm$ 3.21e-04 & 0.07 $\pm$ 3.59e-06 & 0.01 $\pm$ 2.40e-07 \\
\hline 
\end{tabular}
\label{tab:classification_comparison}
\end{table}

\begin{table}[tp]
\tiny
\centering
\caption{Uncertainty term $g_3$ for different choices of $\delta$, for the models trained on CIFAR10.}
\begin{tabular}{|l|c|c|c|c|}
\hline
\multirow{2}{*}{\textbf{Model}} & \multirow{2}{*}{\textbf{K}} & \multicolumn{3}{c|}{ $g_3$} \\
\cline{3-5}
& & $\delta=0.01$ & $\delta=0.05$ & $\delta=0.1$ \\
\hline \hline
\multirow{4}{*}{ResNet18}               & 100 & 0.018 $\pm$ 6.947e-05 & 0.016 $\pm$ 5.587e-05 & 0.015 $\pm$ 5.024e-05     \\
                                        & 500 & 0.158 $\pm$ 1.075e-03 & 0.142 $\pm$ 8.718e-04 & 0.135 $\pm$ 7.889e-04     \\
                                        & 1000 & 0.353 $\pm$ 8.032e-03 & 0.318 $\pm$ 6.490e-03 & 0.302 $\pm$ 5.866e-03 \\
                                        & 10000 & 2.299 $\pm$ 1.568e-01 & 2.076 $\pm$ 1.280e-01 & 1.980 $\pm$ 1.163e-01  \\
\hline \hline
\multirow{4}{*}{ResNet34}               & 100 & 0.023 $\pm$ 1.820e-04 & 0.020 $\pm$ 1.463e-04 & 0.019 $\pm$ 1.316e-04   \\
                                        & 500 & 0.134 $\pm$ 1.667e-03 & 0.120 $\pm$ 1.348e-03 & 0.114 $\pm$ 1.218e-03     \\
                                        & 1000 & 0.292 $\pm$ 2.042e-03 & 0.262 $\pm$ 1.653e-03 & 0.249 $\pm$ 1.496e-03   \\
                                        & 10000 & 2.312 $\pm$ 1.982e-02 & 2.088 $\pm$ 1.614e-02 & 1.991 $\pm$ 1.467e-02 \\
\hline \hline
\multirow{4}{*}{ResNet50}               & 100 & 0.024 $\pm$ 6.475e-04 & 0.021 $\pm$ 5.199e-04 & 0.020 $\pm$ 4.672e-04  \\
                                        & 500 & 0.184 $\pm$ 6.059e-06 & 0.166 $\pm$ 4.785e-06 & 0.157 $\pm$ 4.276e-06   \\
                                        & 1000 & 0.328 $\pm$ 3.415e-03 & 0.295 $\pm$ 2.765e-03 & 0.281 $\pm$ 2.501e-03   \\
                                        & 10000 & 2.256 $\pm$ 1.107e-01 & 2.038 $\pm$ 9.028e-02 & 1.943 $\pm$ 8.206e-02    \\
\hline \hline
\multirow{4}{*}{ShuffleNet (V2\_X1\_0)} & 100 & 0.033 $\pm$ 3.405e-04 & 0.029 $\pm$ 2.739e-04 & 0.028 $\pm$ 2.463e-04  \\
                                        & 500 & 0.182 $\pm$ 2.098e-04 & 0.164 $\pm$ 1.697e-04 & 0.156 $\pm$ 1.534e-04    \\
                                        & 1000 & 0.362 $\pm$ 4.077e-03 & 0.326 $\pm$ 3.296e-03 & 0.310 $\pm$ 2.980e-03   \\
                                        & 10000 & 2.437 $\pm$ 3.086e-02 & 2.201 $\pm$ 2.516e-02 & 2.099 $\pm$ 2.287e-02 \\
\hline \hline
\multirow{4}{*}{ShuffleNet (V2\_X1\_5)} & 100 & 0.039 $\pm$ 3.092e-04 & 0.035 $\pm$ 2.482e-04 & 0.034 $\pm$ 2.230e-04   \\
                                        & 500 & 0.174 $\pm$ 1.115e-03 & 0.157 $\pm$ 9.014e-04 & 0.149 $\pm$ 8.144e-04   \\
                                        & 1000 & 0.349 $\pm$ 5.196e-03 & 0.314 $\pm$ 4.198e-03 & 0.299 $\pm$ 3.794e-03     \\
                                        & 10000 & 2.441 $\pm$ 2.557e-02 & 2.205 $\pm$ 2.086e-02 & 2.102 $\pm$ 1.897e-02   \\
\hline \hline
\multirow{4}{*}{ShuffleNet (V2\_X2\_0)} & 100 & 0.028 $\pm$ 1.264e-04 & 0.025 $\pm$ 1.017e-04 & 0.024 $\pm$ 9.150e-05   \\
                                        & 500 & 0.161 $\pm$ 1.848e-03 & 0.145 $\pm$ 1.493e-03 & 0.138 $\pm$ 1.349e-03    \\
                                        & 1000 & 0.330 $\pm$ 3.463e-03 & 0.297 $\pm$ 2.794e-03 & 0.282 $\pm$ 2.524e-03   \\
                                        & 10000 & 2.484 $\pm$ 3.968e-02 & 2.244 $\pm$ 3.237e-02 & 2.140 $\pm$ 2.944e-02 \\
\hline
\end{tabular}
\label{tab:additional-result-2}
\end{table}

\begin{table}[tp]
\tiny
\centering
\caption{Uncertainty term $g_3$ for different choices of $\delta$ in ImageNet pretrained models.}
\begin{tabular}{|l|c|c|c|c|}
\hline
\multirow{2}{*}{\textbf{Model}} & \multirow{2}{*}{\textbf{K}} & \multicolumn{3}{c|}{ $g_3$} \\
\cline{3-5}
& & $\delta=0.01$ & $\delta=0.05$ & $\delta=0.1$ \\
\hline \hline
ResNet18 V1    & 10000 & 0.315 $\pm$ 0.005 & 0.289 $\pm$ 0.004 & 0.278 $\pm$ 0.004 \\
ResNet34 V1    & 10000 & 0.306 $\pm$ 0.005 & 0.281 $\pm$ 0.004 & 0.270 $\pm$ 0.004 \\
ResNet50 V1    & 10000 & 0.300 $\pm$ 0.004 & 0.275 $\pm$ 0.004 & 0.264 $\pm$ 0.004 \\
ResNet101 V1   & 10000 & 0.296 $\pm$ 0.005 & 0.272 $\pm$ 0.004 & 0.261 $\pm$ 0.004 \\
ResNet152 V1   & 10000 & 0.294 $\pm$ 0.004 & 0.270 $\pm$ 0.004 & 0.259 $\pm$ 0.004 \\
ResNet50 V2    & 10000 & 0.290 $\pm$ 0.004 & 0.266 $\pm$ 0.004 & 0.255 $\pm$ 0.004 \\
ResNet101 V2   & 10000 & 0.286 $\pm$ 0.004 & 0.263 $\pm$ 0.004 & 0.252 $\pm$ 0.004 \\
ResNet152 V2   & 10000 & 0.285 $\pm$ 0.004 & 0.262 $\pm$ 0.004 & 0.251 $\pm$ 0.004 \\
SwinTransformer B     & 10000 & 0.283 $\pm$ 0.004 & 0.259 $\pm$ 0.004 & 0.249 $\pm$ 0.004 \\
SwinTransformer T     & 10000 & 0.289 $\pm$ 0.004 & 0.265 $\pm$ 0.004 & 0.254 $\pm$ 0.004 \\
SwinTransformer V2 B     & 10000 & 0.281 $\pm$ 0.004 & 0.258 $\pm$ 0.004 & 0.247 $\pm$ 0.004 \\
SwinTransformer V2 T     & 10000 & 0.286 $\pm$ 0.004 & 0.262 $\pm$ 0.004 & 0.252 $\pm$ 0.004 \\
VGG13       & 10000 & 0.315 $\pm$ 0.005 & 0.290 $\pm$ 0.005 & 0.278 $\pm$ 0.004 \\
VGG13 BN   & 10000 & 0.312 $\pm$ 0.004 & 0.286 $\pm$ 0.004 & 0.275 $\pm$ 0.004 \\
VGG19       & 10000 & 0.309 $\pm$ 0.005 & 0.283 $\pm$ 0.004 & 0.272 $\pm$ 0.004 \\
VGG19 BN   & 10000 & 0.304 $\pm$ 0.004 & 0.279 $\pm$ 0.004 & 0.268 $\pm$ 0.004 \\
DenseNet121 & 10000 & 0.301 $\pm$ 0.005 & 0.276 $\pm$ 0.004 & 0.266 $\pm$ 0.004 \\
DenseNet161 & 10000 & 0.295 $\pm$ 0.004 & 0.270 $\pm$ 0.004 & 0.260 $\pm$ 0.004 \\
DenseNet169 & 10000 & 0.298 $\pm$ 0.004 & 0.274 $\pm$ 0.004 & 0.263 $\pm$ 0.004 \\
DenseNet201 & 10000 & 0.295 $\pm$ 0.004 & 0.271 $\pm$ 0.004 & 0.260 $\pm$ 0.004 \\
\hline
\end{tabular}
\label{tab:additional-result-imagenet-uncertainty}
\end{table}

Table \ref{tab:additional-result-2} provides the uncertainty term $g_3$ in different settings. It is evident that when the input space is divided into a smaller number of areas, $g_3$ decreases. However, the range of $g_3$ remains substantial in some cases with large $K$ and small training sets. We also provide the uncertainty for the ImageNet models in Table \ref{tab:additional-result-imagenet-uncertainty}.

\subsection{SVM and AdaBoost}
We next want to see how well our bounds can reflect the performance of the models learned by some classical methods, including SVM and AdaBoost. We also used CIFAR-10 and SVHN datasets in this evaluation. 

\textit{Settings:} We evaluated the models returned by AdaBoost with five different hyperparameter configurations: \{n\_estimators: 50, learning\_rate: 1.0\}, \{n\_estimators: 100, learning\_rate: 1.0\}, \{n\_estimators: 50, learning\_rate: 0.5\}, \{n\_estimators: 100, learning\_rate: 0.5\}, and \{n\_estimators: 200, learning\_rate: 0.1\}. For SVM, the hyperparameter sets were: \{C: 0.1, max\_iter: 1000, tol: $1e^{-4}$\}, \{C: 1.0, max\_iter: 1000, tol: $1e^{-4}$\}, \{C: 10.0, max\_iter: 1000, tol: $1e^{-4}$\}, \{C: 1.0, max\_iter: 2000, tol: $1e^{-4}$\}, and \{C: 1.0, max\_iter: 1000, tol: $1e^{-3}$\}. The setup for computing the bounds is the same as before.

\textit{Results:} Table \ref{tab:results-SVM-AdaBoost} reports results for \textit{Rob}, \textit{LocalRob}, \textit{LocalSen}, and \textit{LocalAvg}. We can observe that those quantities can slightly decrease as $K$ increases. In this investigation, \textit{LocalSen} is nonvacuous and \textit{LocalAvg} is often high. The main reason comes from the quality of the trained models. The second column indicates that those models are really bad, and hence their training error can be high. However, \textit{LocalAvg} can reflect the true error of the trained models very well.

\begin{table}[h!]
    \tiny
    \centering
    \caption{Estimates for the true errors of different models, as $K$  varies.}
    \begin{tabular}{|c|c|c|c|c|c|c|}
    \hline
    \multirow{2}{*}{\textbf{Model}} & \multirow{2}{*}{\textbf{Valid Acc}} & \multirow{2}{*}{\textbf{K}} & \multirow{2}{*}{\textbf{Rob}} & \multirow{2}{*}{\textbf{LocalRob}} & \multirow{2}{*}{\textbf{LocalSen}} & \multirow{2}{*}{\textbf{LocalAvg}} \\
    & & & & & &\\
    \hline \hline
    
                \multirow{20}{*}{AdaBoost (for CIFAR10)} & \multirow{4}{*}{25.37}  
                        & 100 & 1.74 $\pm$ 0.00e+00 & 1.74 $\pm$ 7.40e-08 & 1.11 $\pm$ 2.94e-06 & 0.74 $\pm$ 2.05e-08 \\
                
                &  & 1000 & 1.74 $\pm$ 0.00e+00 & 1.73 $\pm$ 1.24e-06 & 1.11 $\pm$ 1.95e-06 & 0.74 $\pm$ 6.18e-08 \\
                
                &  & 5000 & 1.74 $\pm$ 0.00e+00 & 1.69 $\pm$ 8.50e-06 & 1.10 $\pm$ 4.43e-06 & 0.74 $\pm$ 3.76e-07 \\
                
                \cline{2-7}
                 & \multirow{4}{*}{28.31}  
                        & 100 & 1.71 $\pm$ 0.00e+00 & 1.71 $\pm$ 3.09e-08 & 1.11 $\pm$ 3.74e-06 & 0.71 $\pm$ 1.30e-08 \\
                
                &  & 1000 & 1.71 $\pm$ 0.00e+00 & 1.70 $\pm$ 1.20e-06 & 1.10 $\pm$ 1.49e-06 & 0.71 $\pm$ 7.15e-08 \\
                
                &  & 5000 & 1.71 $\pm$ 0.00e+00 & 1.67 $\pm$ 2.85e-06 & 1.09 $\pm$ 6.26e-06 & 0.71 $\pm$ 8.25e-08 \\
                
                \cline{2-7}
                 & \multirow{4}{*}{22.79}  
                        & 100 & 1.77 $\pm$ 0.00e+00 & 1.77 $\pm$ 1.28e-07 & 1.12 $\pm$ 2.92e-06 & 0.77 $\pm$ 4.69e-09 \\
                
                &  & 1000 & 1.77 $\pm$ 0.00e+00 & 1.76 $\pm$ 2.82e-06 & 1.11 $\pm$ 2.04e-06 & 0.77 $\pm$ 3.23e-08 \\
                
                &  & 5000 & 1.77 $\pm$ 0.00e+00 & 1.72 $\pm$ 6.82e-06 & 1.10 $\pm$ 1.10e-05 & 0.77 $\pm$ 1.25e-07 \\
                
                \cline{2-7}
                 & \multirow{4}{*}{26.60}  
                        & 100 & 1.73 $\pm$ 0.00e+00 & 1.73 $\pm$ 6.78e-08 & 1.11 $\pm$ 4.89e-06 & 0.73 $\pm$ 7.84e-09 \\
                
                &  & 1000 & 1.73 $\pm$ 0.00e+00 & 1.73 $\pm$ 1.66e-06 & 1.10 $\pm$ 2.78e-06 & 0.74 $\pm$ 1.86e-08 \\
                
                &  & 5000 & 1.73 $\pm$ 0.00e+00 & 1.69 $\pm$ 5.30e-06 & 1.09 $\pm$ 1.20e-05 & 0.74 $\pm$ 2.24e-07 \\
                
                \cline{2-7}
                 & \multirow{4}{*}{25.03}  
                        & 100 & 1.75 $\pm$ 0.00e+00 & 1.75 $\pm$ 1.92e-07 & 1.12 $\pm$ 1.40e-06 & 0.75 $\pm$ 1.59e-08 \\
                
                &  & 1000 & 1.75 $\pm$ 0.00e+00 & 1.74 $\pm$ 2.05e-06 & 1.11 $\pm$ 4.27e-06 & 0.75 $\pm$ 7.04e-08 \\
                
                &  & 5000 & 1.75 $\pm$ 0.00e+00 & 1.70 $\pm$ 2.09e-06 & 1.10 $\pm$ 1.09e-05 & 0.75 $\pm$ 1.24e-07 \\
                
                \hline
                \multirow{20}{*}{SVM  (for CIFAR10)} & \multirow{4}{*}{24.40}  
                        & 100 & 1.69 $\pm$ 0.00e+00 & 1.69 $\pm$ 2.47e-08 & 1.08 $\pm$ 1.16e-05 & 0.70 $\pm$ 4.53e-09 \\
                
                &  & 1000 & 1.69 $\pm$ 0.00e+00 & 1.69 $\pm$ 1.10e-06 & 1.07 $\pm$ 4.16e-06 & 0.70 $\pm$ 9.96e-08 \\
                
                &  & 5000 & 1.69 $\pm$ 0.00e+00 & 1.65 $\pm$ 4.37e-06 & 1.05 $\pm$ 7.33e-06 & 0.70 $\pm$ 3.08e-07 \\
                
                \cline{2-7}
                & \multirow{4}{*}{24.14}  
                        & 100 & 1.68 $\pm$ 0.00e+00 & 1.68 $\pm$ 9.95e-09 & 1.07 $\pm$ 5.81e-06 & 0.69 $\pm$ 1.62e-08 \\
                
                &  & 1000 & 1.68 $\pm$ 0.00e+00 & 1.67 $\pm$ 1.78e-06 & 1.06 $\pm$ 1.15e-06 & 0.69 $\pm$ 9.01e-08 \\
                
                &  & 5000 & 1.68 $\pm$ 0.00e+00 & 1.64 $\pm$ 2.78e-06 & 1.05 $\pm$ 7.06e-06 & 0.69 $\pm$ 1.78e-07 \\
                
                \cline{2-7}
                & \multirow{4}{*}{23.93}  
                        & 100 & 1.67 $\pm$ 0.00e+00 & 1.67 $\pm$ 2.88e-08 & 1.07 $\pm$ 1.10e-05 & 0.69 $\pm$ 1.41e-08 \\
                
                &  & 1000 & 1.67 $\pm$ 0.00e+00 & 1.67 $\pm$ 2.09e-07 & 1.06 $\pm$ 3.82e-06 & 0.69 $\pm$ 4.41e-08 \\
                
                &  & 5000 & 1.67 $\pm$ 0.00e+00 & 1.64 $\pm$ 5.54e-06 & 1.05 $\pm$ 2.88e-06 & 0.68 $\pm$ 1.10e-07 \\
                
                \cline{2-7}
                 & \multirow{4}{*}{24.14}  
                        & 100 & 1.68 $\pm$ 0.00e+00 & 1.68 $\pm$ 1.74e-08 & 1.07 $\pm$ 4.44e-06 & 0.69 $\pm$ 2.35e-08 \\
                
                &  & 1000 & 1.68 $\pm$ 0.00e+00 & 1.67 $\pm$ 6.94e-07 & 1.06 $\pm$ 2.08e-06 & 0.69 $\pm$ 5.67e-08 \\
                
                &  & 5000 & 1.68 $\pm$ 0.00e+00 & 1.64 $\pm$ 6.59e-06 & 1.05 $\pm$ 6.55e-06 & 0.69 $\pm$ 8.82e-08 \\
                
                \cline{2-7}
                & \multirow{4}{*}{24.05}  
                        & 100 & 1.68 $\pm$ 0.00e+00 & 1.68 $\pm$ 6.40e-08 & 1.07 $\pm$ 2.56e-06 & 0.69 $\pm$ 1.63e-08 \\
                
                &  & 1000 & 1.68 $\pm$ 0.00e+00 & 1.67 $\pm$ 3.15e-07 & 1.06 $\pm$ 2.45e-06 & 0.69 $\pm$ 5.78e-08 \\
                
                &  & 5000 & 1.68 $\pm$ 0.00e+00 & 1.64 $\pm$ 2.82e-06 & 1.05 $\pm$ 5.82e-06 & 0.69 $\pm$ 1.08e-07 \\
                
                \hline
                \multirow{20}{*}{AdaBoost (for SVHN)} & \multirow{4}{*}{20.54}  
                        & 100 & 1.81 $\pm$ 0.00e+00 & 1.81 $\pm$ 1.26e-07 & 1.11 $\pm$ 8.10e-06 & 0.80 $\pm$ 9.32e-07 \\
                
                &  & 1000 & 1.81 $\pm$ 0.00e+00 & 1.79 $\pm$ 4.09e-06 & 1.09 $\pm$ 9.33e-06 & 0.80 $\pm$ 6.07e-07 \\
                
                &  & 5000 & 1.81 $\pm$ 0.00e+00 & 1.72 $\pm$ 1.84e-05 & 1.07 $\pm$ 5.22e-06 & 0.80 $\pm$ 1.49e-07 \\
                
                \cline{2-7}
                & \multirow{4}{*}{20.71}  
                        & 100 & 1.80 $\pm$ 0.00e+00 & 1.80 $\pm$ 1.07e-06 & 1.11 $\pm$ 3.44e-05 & 0.80 $\pm$ 1.21e-06 \\
                
                &  & 1000 & 1.80 $\pm$ 0.00e+00 & 1.79 $\pm$ 8.45e-06 & 1.09 $\pm$ 8.88e-06 & 0.80 $\pm$ 4.00e-07 \\
                
                &  & 5000 & 1.80 $\pm$ 0.00e+00 & 1.72 $\pm$ 2.75e-05 & 1.07 $\pm$ 5.66e-06 & 0.80 $\pm$ 1.97e-07 \\
                
                \cline{2-7}
                 & \multirow{4}{*}{19.59}  
                        & 100 & 1.81 $\pm$ 0.00e+00 & 1.81 $\pm$ 7.19e-07 & 1.11 $\pm$ 3.14e-05 & 0.81 $\pm$ 1.16e-06 \\
                
                &  & 1000 & 1.81 $\pm$ 0.00e+00 & 1.79 $\pm$ 1.12e-05 & 1.09 $\pm$ 1.00e-05 & 0.81 $\pm$ 5.54e-07 \\
                
                &  & 5000 & 1.81 $\pm$ 0.00e+00 & 1.70 $\pm$ 3.69e-05 & 1.06 $\pm$ 7.38e-06 & 0.81 $\pm$ 2.82e-07 \\
                
                \cline{2-7}
                 & \multirow{4}{*}{19.61}  
                        & 100 & 1.81 $\pm$ 0.00e+00 & 1.81 $\pm$ 6.33e-07 & 1.11 $\pm$ 2.39e-05 & 0.81 $\pm$ 1.11e-06 \\
                
                &  & 1000 & 1.81 $\pm$ 0.00e+00 & 1.79 $\pm$ 1.84e-05 & 1.09 $\pm$ 2.12e-05 & 0.81 $\pm$ 5.56e-07 \\
                
                &  & 5000 & 1.81 $\pm$ 0.00e+00 & 1.70 $\pm$ 8.42e-05 & 1.07 $\pm$ 2.45e-06 & 0.81 $\pm$ 3.58e-07 \\
                
                \cline{2-7}
                 & \multirow{4}{*}{19.59}  
                        & 100 & 1.81 $\pm$ 0.00e+00 & 1.81 $\pm$ 5.35e-07 & 1.11 $\pm$ 3.07e-05 & 0.81 $\pm$ 1.10e-06 \\
                
                &  & 1000 & 1.81 $\pm$ 0.00e+00 & 1.79 $\pm$ 1.95e-05 & 1.09 $\pm$ 9.70e-06 & 0.81 $\pm$ 2.94e-07 \\
                
                &  & 5000 & 1.81 $\pm$ 0.00e+00 & 1.70 $\pm$ 2.35e-05 & 1.07 $\pm$ 1.89e-06 & 0.81 $\pm$ 1.98e-07 \\
                
                \hline
                \multirow{20}{*}{SVM  (for SVHN)} & \multirow{4}{*}{24.58}  
                        & 100 & 1.69 $\pm$ 0.00e+00 & 1.69 $\pm$ 7.32e-09 & 1.09 $\pm$ 4.64e-06 & 0.70 $\pm$ 4.71e-07 \\
                
                &  & 1000 & 1.69 $\pm$ 0.00e+00 & 1.69 $\pm$ 3.64e-07 & 1.08 $\pm$ 1.60e-06 & 0.70 $\pm$ 3.78e-07 \\
                
                &  & 5000 & 1.69 $\pm$ 0.00e+00 & 1.67 $\pm$ 2.99e-06 & 1.07 $\pm$ 2.56e-06 & 0.70 $\pm$ 8.82e-08 \\
                
                \cline{2-7}
                 & \multirow{4}{*}{23.78}  
                        & 100 & 1.67 $\pm$ 0.00e+00 & 1.67 $\pm$ 2.19e-09 & 1.08 $\pm$ 1.09e-06 & 0.69 $\pm$ 4.53e-07 \\
                
                &  & 1000 & 1.67 $\pm$ 0.00e+00 & 1.67 $\pm$ 2.38e-07 & 1.07 $\pm$ 1.74e-06 & 0.69 $\pm$ 1.12e-07 \\
                
                &  & 5000 & 1.67 $\pm$ 0.00e+00 & 1.65 $\pm$ 1.49e-06 & 1.06 $\pm$ 1.21e-06 & 0.69 $\pm$ 4.65e-08 \\
                
                \cline{2-7}
                 & \multirow{4}{*}{22.87}  
                        & 100 & 1.66 $\pm$ 0.00e+00 & 1.66 $\pm$ 7.24e-09 & 1.07 $\pm$ 7.85e-07 & 0.69 $\pm$ 6.38e-07 \\
                
                &  & 1000 & 1.66 $\pm$ 0.00e+00 & 1.66 $\pm$ 2.12e-07 & 1.06 $\pm$ 3.15e-06 & 0.68 $\pm$ 2.06e-07 \\
                
                &  & 5000 & 1.66 $\pm$ 0.00e+00 & 1.64 $\pm$ 5.68e-07 & 1.05 $\pm$ 2.12e-06 & 0.68 $\pm$ 3.44e-07 \\
                
                \cline{2-7}
                & \multirow{4}{*}{23.78}  
                        & 100 & 1.67 $\pm$ 0.00e+00 & 1.67 $\pm$ 8.40e-09 & 1.08 $\pm$ 1.58e-06 & 0.69 $\pm$ 3.94e-07 \\
                
                &  & 1000 & 1.67 $\pm$ 0.00e+00 & 1.67 $\pm$ 9.57e-08 & 1.07 $\pm$ 2.97e-06 & 0.69 $\pm$ 2.70e-07 \\
                
                &  & 5000 & 1.67 $\pm$ 0.00e+00 & 1.65 $\pm$ 1.89e-06 & 1.06 $\pm$ 1.77e-06 & 0.69 $\pm$ 1.80e-07 \\
                
                \cline{2-7}
                 & \multirow{4}{*}{24.38}  
                        & 100 & 1.68 $\pm$ 0.00e+00 & 1.68 $\pm$ 4.59e-09 & 1.08 $\pm$ 2.33e-06 & 0.70 $\pm$ 3.47e-07 \\
                
                &  & 1000 & 1.68 $\pm$ 0.00e+00 & 1.68 $\pm$ 3.25e-07 & 1.08 $\pm$ 2.93e-06 & 0.70 $\pm$ 2.37e-07 \\
                
                &  & 5000 & 1.68 $\pm$ 0.00e+00 & 1.66 $\pm$ 1.95e-06 & 1.07 $\pm$ 1.87e-06 & 0.69 $\pm$ 2.10e-07 \\
                
                
    \hline
    \end{tabular}
    \label{tab:results-SVM-AdaBoost}
    \end{table}

%% file: RobGen.bbl

\begin{thebibliography}{31}
\ifx \bisbn   \undefined \def \bisbn  #1{ISBN #1}\fi
\ifx \binits  \undefined \def \binits#1{#1}\fi
\ifx \bauthor  \undefined \def \bauthor#1{#1}\fi
\ifx \batitle  \undefined \def \batitle#1{#1}\fi
\ifx \bjtitle  \undefined \def \bjtitle#1{#1}\fi
\ifx \bvolume  \undefined \def \bvolume#1{\textbf{#1}}\fi
\ifx \byear  \undefined \def \byear#1{#1}\fi
\ifx \bissue  \undefined \def \bissue#1{#1}\fi
\ifx \bfpage  \undefined \def \bfpage#1{#1}\fi
\ifx \blpage  \undefined \def \blpage #1{#1}\fi
\ifx \burl  \undefined \def \burl#1{\textsf{#1}}\fi
\ifx \doiurl  \undefined \def \doiurl#1{\url{https://doi.org/#1}}\fi
\ifx \betal  \undefined \def \betal{\textit{et al.}}\fi
\ifx \binstitute  \undefined \def \binstitute#1{#1}\fi
\ifx \binstitutionaled  \undefined \def \binstitutionaled#1{#1}\fi
\ifx \bctitle  \undefined \def \bctitle#1{#1}\fi
\ifx \beditor  \undefined \def \beditor#1{#1}\fi
\ifx \bpublisher  \undefined \def \bpublisher#1{#1}\fi
\ifx \bbtitle  \undefined \def \bbtitle#1{#1}\fi
\ifx \bedition  \undefined \def \bedition#1{#1}\fi
\ifx \bseriesno  \undefined \def \bseriesno#1{#1}\fi
\ifx \blocation  \undefined \def \blocation#1{#1}\fi
\ifx \bsertitle  \undefined \def \bsertitle#1{#1}\fi
\ifx \bsnm \undefined \def \bsnm#1{#1}\fi
\ifx \bsuffix \undefined \def \bsuffix#1{#1}\fi
\ifx \bparticle \undefined \def \bparticle#1{#1}\fi
\ifx \barticle \undefined \def \barticle#1{#1}\fi
\bibcommenthead
\ifx \bconfdate \undefined \def \bconfdate #1{#1}\fi
\ifx \botherref \undefined \def \botherref #1{#1}\fi
\ifx \url \undefined \def \url#1{\textsf{#1}}\fi
\ifx \bchapter \undefined \def \bchapter#1{#1}\fi
\ifx \bbook \undefined \def \bbook#1{#1}\fi
\ifx \bcomment \undefined \def \bcomment#1{#1}\fi
\ifx \oauthor \undefined \def \oauthor#1{#1}\fi
\ifx \citeauthoryear \undefined \def \citeauthoryear#1{#1}\fi
\ifx \endbibitem  \undefined \def \endbibitem {}\fi
\ifx \bconflocation  \undefined \def \bconflocation#1{#1}\fi
\ifx \arxivurl  \undefined \def \arxivurl#1{\textsf{#1}}\fi
\csname PreBibitemsHook\endcsname

\bibitem[\protect\citeauthoryear{Xu and
  Mannor}{2012}]{xu2012robustnessGeneralize}
\begin{barticle}
\bauthor{\bsnm{Xu}, \binits{H.}},
\bauthor{\bsnm{Mannor}, \binits{S.}}:
\batitle{Robustness and generalization}.
\bjtitle{Machine learning}
\bvolume{86}(\bissue{3}),
\bfpage{391}--\blpage{423}
(\byear{2012})
\end{barticle}
\endbibitem

\bibitem[\protect\citeauthoryear{Madry et~al.}{2018}]{madry2018towards}
\begin{bchapter}
\bauthor{\bsnm{Madry}, \binits{A.}},
\bauthor{\bsnm{Makelov}, \binits{A.}},
\bauthor{\bsnm{Schmidt}, \binits{L.}},
\bauthor{\bsnm{Tsipras}, \binits{D.}},
\bauthor{\bsnm{Vladu}, \binits{A.}}:
\bctitle{Towards deep learning models resistant to adversarial attacks}.
In: \bbtitle{International Conference on Learning Representations}
(\byear{2018})
\end{bchapter}
\endbibitem

\bibitem[\protect\citeauthoryear{Xu et~al.}{2020}]{xu2020adversarial}
\begin{barticle}
\bauthor{\bsnm{Xu}, \binits{H.}},
\bauthor{\bsnm{Ma}, \binits{Y.}},
\bauthor{\bsnm{Liu}, \binits{H.-C.}},
\bauthor{\bsnm{Deb}, \binits{D.}},
\bauthor{\bsnm{Liu}, \binits{H.}},
\bauthor{\bsnm{Tang}, \binits{J.-L.}},
\bauthor{\bsnm{Jain}, \binits{A.K.}}:
\batitle{Adversarial attacks and defenses in images, graphs and text: A
  review}.
\bjtitle{International Journal of Automation and Computing}
\bvolume{17},
\bfpage{151}--\blpage{178}
(\byear{2020})
\end{barticle}
\endbibitem

\bibitem[\protect\citeauthoryear{Zhou et~al.}{2022}]{zhou2022adversarial}
\begin{barticle}
\bauthor{\bsnm{Zhou}, \binits{S.}},
\bauthor{\bsnm{Liu}, \binits{C.}},
\bauthor{\bsnm{Ye}, \binits{D.}},
\bauthor{\bsnm{Zhu}, \binits{T.}},
\bauthor{\bsnm{Zhou}, \binits{W.}},
\bauthor{\bsnm{Yu}, \binits{P.S.}}:
\batitle{Adversarial attacks and defenses in deep learning: From a perspective
  of cybersecurity}.
\bjtitle{ACM Computing Surveys}
\bvolume{55}(\bissue{8}),
\bfpage{1}--\blpage{39}
(\byear{2022})
\end{barticle}
\endbibitem

\bibitem[\protect\citeauthoryear{Sokolic
  et~al.}{2017}]{sokolic2017generalization}
\begin{bchapter}
\bauthor{\bsnm{Sokolic}, \binits{J.}},
\bauthor{\bsnm{Giryes}, \binits{R.}},
\bauthor{\bsnm{Sapiro}, \binits{G.}},
\bauthor{\bsnm{Rodrigues}, \binits{M.}}:
\bctitle{Generalization error of invariant classifiers}.
In: \bbtitle{Artificial Intelligence and Statistics},
pp. \bfpage{1094}--\blpage{1103}
(\byear{2017}).
\bcomment{PMLR}
\end{bchapter}
\endbibitem

\bibitem[\protect\citeauthoryear{Sokoli{\'c}
  et~al.}{2017}]{sokolic2017robustDNN}
\begin{barticle}
\bauthor{\bsnm{Sokoli{\'c}}, \binits{J.}},
\bauthor{\bsnm{Giryes}, \binits{R.}},
\bauthor{\bsnm{Sapiro}, \binits{G.}},
\bauthor{\bsnm{Rodrigues}, \binits{M.R.}}:
\batitle{Robust large margin deep neural networks}.
\bjtitle{IEEE Transactions on Signal Processing}
\bvolume{65}(\bissue{16}),
\bfpage{4265}--\blpage{4280}
(\byear{2017})
\end{barticle}
\endbibitem

\bibitem[\protect\citeauthoryear{Qi et~al.}{2013}]{qi2013robustSVM}
\begin{barticle}
\bauthor{\bsnm{Qi}, \binits{Z.}},
\bauthor{\bsnm{Tian}, \binits{Y.}},
\bauthor{\bsnm{Shi}, \binits{Y.}}:
\batitle{Robust twin support vector machine for pattern classification}.
\bjtitle{Pattern Recognition}
\bvolume{46}(\bissue{1}),
\bfpage{305}--\blpage{316}
(\byear{2013})
\end{barticle}
\endbibitem

\bibitem[\protect\citeauthoryear{Bellet and
  Habrard}{2015}]{bellet2015robustnessMetricL}
\begin{barticle}
\bauthor{\bsnm{Bellet}, \binits{A.}},
\bauthor{\bsnm{Habrard}, \binits{A.}}:
\batitle{Robustness and generalization for metric learning}.
\bjtitle{Neurocomputing}
\bvolume{151},
\bfpage{259}--\blpage{267}
(\byear{2015})
\end{barticle}
\endbibitem

\bibitem[\protect\citeauthoryear{Liu et~al.}{2017}]{liu2017spectralClustering}
\begin{barticle}
\bauthor{\bsnm{Liu}, \binits{H.}},
\bauthor{\bsnm{Wu}, \binits{J.}},
\bauthor{\bsnm{Liu}, \binits{T.}},
\bauthor{\bsnm{Tao}, \binits{D.}},
\bauthor{\bsnm{Fu}, \binits{Y.}}:
\batitle{Spectral ensemble clustering via weighted k-means: Theoretical and
  practical evidence}.
\bjtitle{IEEE Transactions on Knowledge and Data Engineering}
\bvolume{29}(\bissue{5}),
\bfpage{1129}--\blpage{1143}
(\byear{2017})
\end{barticle}
\endbibitem

\bibitem[\protect\citeauthoryear{Li et~al.}{2021}]{li2021orthogonalDNN}
\begin{barticle}
\bauthor{\bsnm{Li}, \binits{S.}},
\bauthor{\bsnm{Jia}, \binits{K.}},
\bauthor{\bsnm{Wen}, \binits{Y.}},
\bauthor{\bsnm{Liu}, \binits{T.}},
\bauthor{\bsnm{Tao}, \binits{D.}}:
\batitle{Orthogonal deep neural networks}.
\bjtitle{IEEE Transactions on Pattern Analysis \& Machine Intelligence}
\bvolume{43}(\bissue{04}),
\bfpage{1352}--\blpage{1368}
(\byear{2021})
\end{barticle}
\endbibitem

\bibitem[\protect\citeauthoryear{Shi et~al.}{2014}]{shi2014sparse}
\begin{bchapter}
\bauthor{\bsnm{Shi}, \binits{Y.}},
\bauthor{\bsnm{Bellet}, \binits{A.}},
\bauthor{\bsnm{Sha}, \binits{F.}}:
\bctitle{Sparse compositional metric learning}.
In: \bbtitle{Proceedings of the AAAI Conference on Artificial Intelligence},
vol. \bseriesno{28}
(\byear{2014})
\end{bchapter}
\endbibitem

\bibitem[\protect\citeauthoryear{Kawaguchi
  et~al.}{2022}]{kawaguchi2022robustness}
\begin{bchapter}
\bauthor{\bsnm{Kawaguchi}, \binits{K.}},
\bauthor{\bsnm{Deng}, \binits{Z.}},
\bauthor{\bsnm{Luh}, \binits{K.}},
\bauthor{\bsnm{Huang}, \binits{J.}}:
\bctitle{Robustness implies generalization via data-dependent generalization
  bounds}.
In: \bbtitle{International Conference on Machine Learning},
pp. \bfpage{10866}--\blpage{10894}
(\byear{2022}).
\bcomment{PMLR}
\end{bchapter}
\endbibitem

\bibitem[\protect\citeauthoryear{Hou et~al.}{2023}]{hou2023instanceGen}
\begin{barticle}
\bauthor{\bsnm{Hou}, \binits{S.}},
\bauthor{\bsnm{Kassraie}, \binits{P.}},
\bauthor{\bsnm{Kratsios}, \binits{A.}},
\bauthor{\bsnm{Krause}, \binits{A.}},
\bauthor{\bsnm{Rothfuss}, \binits{J.}}:
\batitle{Instance-dependent generalization bounds via optimal transport}.
\bjtitle{Journal of Machine Learning Research}
\bvolume{24},
\bfpage{1}--\blpage{50}
(\byear{2023})
\end{barticle}
\endbibitem

\bibitem[\protect\citeauthoryear{Bartlett and Mendelson}{2002}]{bartlett2002RC}
\begin{barticle}
\bauthor{\bsnm{Bartlett}, \binits{P.L.}},
\bauthor{\bsnm{Mendelson}, \binits{S.}}:
\batitle{Rademacher and gaussian complexities: Risk bounds and structural
  results}.
\bjtitle{Journal of Machine Learning Research}
\bvolume{3}(\bissue{Nov}),
\bfpage{463}--\blpage{482}
(\byear{2002})
\end{barticle}
\endbibitem

\bibitem[\protect\citeauthoryear{Golowich et~al.}{2020}]{golowich2020RC}
\begin{barticle}
\bauthor{\bsnm{Golowich}, \binits{N.}},
\bauthor{\bsnm{Rakhlin}, \binits{A.}},
\bauthor{\bsnm{Shamir}, \binits{O.}}:
\batitle{Size-independent sample complexity of neural networks}.
\bjtitle{Information and Inference: A Journal of the IMA}
\bvolume{9}(\bissue{2}),
\bfpage{473}--\blpage{504}
(\byear{2020})
\end{barticle}
\endbibitem

\bibitem[\protect\citeauthoryear{Shalev-Shwartz
  et~al.}{2010}]{shalev2010StabilityLearnability}
\begin{barticle}
\bauthor{\bsnm{Shalev-Shwartz}, \binits{S.}},
\bauthor{\bsnm{Shamir}, \binits{O.}},
\bauthor{\bsnm{Srebro}, \binits{N.}},
\bauthor{\bsnm{Sridharan}, \binits{K.}}:
\batitle{Learnability, stability and uniform convergence}.
\bjtitle{The Journal of Machine Learning Research}
\bvolume{11},
\bfpage{2635}--\blpage{2670}
(\byear{2010})
\end{barticle}
\endbibitem

\bibitem[\protect\citeauthoryear{Feldman and
  Vondrak}{2019}]{feldman2019stability}
\begin{bchapter}
\bauthor{\bsnm{Feldman}, \binits{V.}},
\bauthor{\bsnm{Vondrak}, \binits{J.}}:
\bctitle{High probability generalization bounds for uniformly stable algorithms
  with nearly optimal rate}.
In: \bbtitle{Conference on Learning Theory},
pp. \bfpage{1270}--\blpage{1279}
(\byear{2019}).
\bcomment{PMLR}
\end{bchapter}
\endbibitem

\bibitem[\protect\citeauthoryear{McAllester}{1999}]{mcallester1999PACBayes}
\begin{barticle}
\bauthor{\bsnm{McAllester}, \binits{D.A.}}:
\batitle{Some pac-bayesian theorems}.
\bjtitle{Machine Learning}
\bvolume{37}(\bissue{3}),
\bfpage{355}--\blpage{363}
(\byear{1999})
\end{barticle}
\endbibitem

\bibitem[\protect\citeauthoryear{Haddouche and Guedj}{2023}]{haddouche2023pac}
\begin{botherref}
\oauthor{\bsnm{Haddouche}, \binits{M.}},
\oauthor{\bsnm{Guedj}, \binits{B.}}:
Pac-bayes generalisation bounds for heavy-tailed losses through
  supermartingales.
Transactions on Machine Learning Research
(2023)
\end{botherref}
\endbibitem

\bibitem[\protect\citeauthoryear{Biggs and Guedj}{2023}]{biggs2023tighterPAC}
\begin{bchapter}
\bauthor{\bsnm{Biggs}, \binits{F.}},
\bauthor{\bsnm{Guedj}, \binits{B.}}:
\bctitle{Tighter pac-bayes generalisation bounds by leveraging example
  difficulty}.
In: \bbtitle{International Conference on Artificial Intelligence and
  Statistics},
pp. \bfpage{8165}--\blpage{8182}
(\byear{2023}).
\bcomment{PMLR}
\end{bchapter}
\endbibitem

\bibitem[\protect\citeauthoryear{Zhou et~al.}{2019}]{zhou2019CompressionBound}
\begin{bchapter}
\bauthor{\bsnm{Zhou}, \binits{W.}},
\bauthor{\bsnm{Veitch}, \binits{V.}},
\bauthor{\bsnm{Austern}, \binits{M.}},
\bauthor{\bsnm{Adams}, \binits{R.P.}},
\bauthor{\bsnm{Orbanz}, \binits{P.}}:
\bctitle{Non-vacuous generalization bounds at the imagenet scale: a
  pac-bayesian compression approach}.
In: \bbtitle{International Conference on Learning Representations (ICLR)}
(\byear{2019})
\end{bchapter}
\endbibitem

\bibitem[\protect\citeauthoryear{Arora et~al.}{2018}]{arora2018strongerBounds}
\begin{bchapter}
\bauthor{\bsnm{Arora}, \binits{S.}},
\bauthor{\bsnm{Ge}, \binits{R.}},
\bauthor{\bsnm{Neyshabur}, \binits{B.}},
\bauthor{\bsnm{Zhang}, \binits{Y.}}:
\bctitle{Stronger generalization bounds for deep nets via a compression
  approach}.
In: \bbtitle{International Conference on Machine Learning},
pp. \bfpage{254}--\blpage{263}
(\byear{2018}).
\bcomment{PMLR}
\end{bchapter}
\endbibitem

\bibitem[\protect\citeauthoryear{Bartlett
  et~al.}{2017}]{bartlett2017SpectralMarginDNN}
\begin{barticle}
\bauthor{\bsnm{Bartlett}, \binits{P.L.}},
\bauthor{\bsnm{Foster}, \binits{D.J.}},
\bauthor{\bsnm{Telgarsky}, \binits{M.J.}}:
\batitle{Spectrally-normalized margin bounds for neural networks}.
\bjtitle{Advances in Neural Information Processing Systems}
\bvolume{30},
\bfpage{6240}--\blpage{6249}
(\byear{2017})
\end{barticle}
\endbibitem

\bibitem[\protect\citeauthoryear{Wei and Ma}{2019}]{wei2019dataRC}
\begin{bchapter}
\bauthor{\bsnm{Wei}, \binits{C.}},
\bauthor{\bsnm{Ma}, \binits{T.}}:
\bctitle{Data-dependent sample complexity of deep neural networks via lipschitz
  augmentation}.
In: \bbtitle{Advances in Neural Information Processing Systems},
vol. \bseriesno{32}
(\byear{2019})
\end{bchapter}
\endbibitem

\bibitem[\protect\citeauthoryear{Mustafa et~al.}{2024}]{mustafa2024non}
\begin{bchapter}
\bauthor{\bsnm{Mustafa}, \binits{W.}},
\bauthor{\bsnm{Liznerski}, \binits{P.}},
\bauthor{\bsnm{Ledent}, \binits{A.}},
\bauthor{\bsnm{Wagner}, \binits{D.}},
\bauthor{\bsnm{Wang}, \binits{P.}},
\bauthor{\bsnm{Kloft}, \binits{M.}}:
\bctitle{Non-vacuous generalization bounds for adversarial risk in stochastic
  neural networks}.
In: \bbtitle{International Conference on Artificial Intelligence and
  Statistics},
pp. \bfpage{4528}--\blpage{4536}
(\byear{2024}).
\bcomment{PMLR}
\end{bchapter}
\endbibitem

\bibitem[\protect\citeauthoryear{Neyshabur
  et~al.}{2018}]{neyshabur2018SpectralMarginDNN}
\begin{bchapter}
\bauthor{\bsnm{Neyshabur}, \binits{B.}},
\bauthor{\bsnm{Bhojanapalli}, \binits{S.}},
\bauthor{\bsnm{Srebro}, \binits{N.}}:
\bctitle{A pac-bayesian approach to spectrally-normalized margin bounds for
  neural networks}.
In: \bbtitle{International Conference on Learning Representations}
(\byear{2018})
\end{bchapter}
\endbibitem

\bibitem[\protect\citeauthoryear{Lei and Ying}{2020}]{lei2020stability}
\begin{bchapter}
\bauthor{\bsnm{Lei}, \binits{Y.}},
\bauthor{\bsnm{Ying}, \binits{Y.}}:
\bctitle{Fine-grained analysis of stability and generalization for stochastic
  gradient descent}.
In: \bbtitle{International Conference on Machine Learning},
pp. \bfpage{5809}--\blpage{5819}
(\byear{2020})
\end{bchapter}
\endbibitem

\bibitem[\protect\citeauthoryear{Biggs and
  Guedj}{2022}]{biggs2022nonvacuousBound}
\begin{bchapter}
\bauthor{\bsnm{Biggs}, \binits{F.}},
\bauthor{\bsnm{Guedj}, \binits{B.}}:
\bctitle{Non-vacuous generalisation bounds for shallow neural networks}.
In: \bbtitle{International Conference on Machine Learning},
pp. \bfpage{1963}--\blpage{1981}
(\byear{2022}).
\bcomment{PMLR}
\end{bchapter}
\endbibitem

\bibitem[\protect\citeauthoryear{Zhang
  et~al.}{2017}]{zhang2017DNNgeneralization}
\begin{bchapter}
\bauthor{\bsnm{Zhang}, \binits{C.}},
\bauthor{\bsnm{Bengio}, \binits{S.}},
\bauthor{\bsnm{Hardt}, \binits{M.}},
\bauthor{\bsnm{Recht}, \binits{B.}},
\bauthor{\bsnm{Vinyals}, \binits{O.}}:
\bctitle{Understanding deep learning requires rethinking generalization}.
In: \bbtitle{International Conference on Learning Representations}
(\byear{2017})
\end{bchapter}
\endbibitem

\bibitem[\protect\citeauthoryear{Qin et~al.}{2019}]{qin2019adversarial}
\begin{bchapter}
\bauthor{\bsnm{Qin}, \binits{C.}},
\bauthor{\bsnm{Martens}, \binits{J.}},
\bauthor{\bsnm{Gowal}, \binits{S.}},
\bauthor{\bsnm{Krishnan}, \binits{D.}},
\bauthor{\bsnm{Dvijotham}, \binits{K.}},
\bauthor{\bsnm{Fawzi}, \binits{A.}},
\bauthor{\bsnm{De}, \binits{S.}},
\bauthor{\bsnm{Stanforth}, \binits{R.}},
\bauthor{\bsnm{Kohli}, \binits{P.}}:
\bctitle{Adversarial robustness through local linearization}.
In: \bbtitle{Advances in Neural Information Processing Systems},
vol. \bseriesno{32}
(\byear{2019})
\end{bchapter}
\endbibitem

\bibitem[\protect\citeauthoryear{Khromov and
  Singh}{2024}]{khromov2024LipschitzNN}
\begin{bchapter}
\bauthor{\bsnm{Khromov}, \binits{G.}},
\bauthor{\bsnm{Singh}, \binits{S.P.}}:
\bctitle{Some intriguing aspects about lipschitz continuity of neural
  networks}.
In: \bbtitle{International Conference on Learning Representations}
(\byear{2024})
\end{bchapter}
\endbibitem

\end{thebibliography}
